\documentclass{article}

     \usepackage[final]{neurips_2020}

\usepackage[utf8]{inputenc} % allow utf-8 input
\usepackage[T1]{fontenc}    % use 8-bit T1 fonts
\usepackage{hyperref}       % hyperlinks
\usepackage{url}            % simple URL typesetting
\usepackage{booktabs}       % professional-quality tables
\usepackage{amsfonts}       % blackboard math symbols
\usepackage{nicefrac}       % compact symbols for 1/2, etc.
\usepackage{microtype}      % microtypography

\usepackage{graphicx}
\usepackage{subfigure}
\usepackage{booktabs} % for professional tables
\usepackage{hyperref}
\usepackage{times}		
\usepackage[T1]{fontenc}
\usepackage[utf8]{inputenc} 
\DeclareUnicodeCharacter{00A0}{~}
\usepackage{amssymb,amsmath,amscd,amsfonts,amsthm,bbm,mathrsfs,yhmath}

\usepackage[shortlabels]{enumitem}
\usepackage{hyperref}
\usepackage{caption}
\usepackage{graphics}
\usepackage{mathtools}
\usepackage{cleveref}
\usepackage{comment}

\usepackage{paralist}                                       
\usepackage{pdfpages}

\theoremstyle{definition}
\newtheorem{theorem}{Theorem}
\newtheorem{lemma}[theorem]{Lemma}
\newtheorem{corollary}[theorem]{Corollary}
\newtheorem{proposition}[theorem]{Proposition}
\newtheorem{definition}{Definition}
\newtheorem{remark}{Remark}

\newcommand{\subscript}[2]{$#1 _ #2$}
\newlist{assumplist}{enumerate}{1}
\setlist[assumplist]{label=(\subscript{\textbf{A}}{{\arabic*}})}
\Crefname{assumplisti}{Assumption}{Assumptions}

\newlist{assumplist2}{enumerate}{1}
\setlist[assumplist2]{label=(\subscript{\textbf{B}}{{\arabic*}})}
\Crefname{assumplist2i}{Assumption}{Assumptions}

\newcommand\numberthis{\addtocounter{equation}{1}\tag{\theequation}}

\usepackage[textwidth=2cm, textsize=footnotesize]{todonotes}  
\newcommand{\aknote}[1]{\todo[color=cyan!20]{#1}}

\newcommand{\cF}{{\mathcal F}}

\newcommand{\cP}{{\mathcal P}}
\newcommand{\X}{{\mathcal X}}

\newcommand{\cH}{{\mathcal H}}

\newcommand{\E}{{{\mathbb E}}}
 
\newcommand{\R}{{\mathbb R}}

\newcommand{\KL}{\mathop{\mathrm{KL}}\nolimits}

\newcommand{\ps}[1]{\langle #1 \rangle}

\title{A Non-Asymptotic Analysis for \\Stein Variational Gradient Descent}

\author{Anna Korba\\
	Gatsby Computational Neuroscience Unit\\
	University College London\\
	a.korba@ucl.ac.uk
  \And
  Adil Salim\\
  Visual Computing Center\\
  KAUST\\
  adil.salim@kaust.edu.sa
  \And Michael Arbel\\
  	Gatsby Computational Neuroscience Unit\\
  University College London\\
  michael.n.arbel@gmail.com
  \And Giulia Luise\\
  Computer Science Department\\
   University College London\\
   g.luise16@ucl.ac.uk
  \And Arthur Gretton\\
  	Gatsby Computational Neuroscience Unit\\
  University College London\\
arthur.gretton@gmail.com
}

\begin{document}

\maketitle

\begin{abstract}
	We study the Stein Variational Gradient Descent (SVGD) algorithm, which optimises a set of particles to approximate a target probability distribution $\pi\propto e^{-V}$ on $\R^d$. In the population limit, SVGD performs gradient descent in the space of probability distributions on the KL divergence with respect to $\pi$, where the gradient is smoothed through a kernel integral operator. In this paper, we provide a novel finite time analysis for the SVGD algorithm. We provide a descent lemma establishing that the algorithm decreases the objective at each iteration, and rates of convergence for the average Stein Fisher divergence (also referred to as Kernel Stein Discrepancy). We also provide a convergence result of the finite particle system corresponding to the practical implementation of SVGD to its population version.
\end{abstract}

%\vspace{-2mm}
\section{Introduction}
%\vspace{-2mm}

The task of sampling from a target distribution is common in Bayesian inference, where the distribution of interest is the posterior distribution of the parameters. Unfortunately, the posterior distribution is generally difficult to compute due to the presence of an intractable integral. This sampling problem can be formulated from an optimization point of view \citep{wibisono2018sampling}. We assume that the target distribution $\pi$ admits a density proportional to $\exp(-V)$ with respect to Lebesgue measure over
$\X= \R^d $, where $V:\X\to \R$ is referred to as the potential function. In this setting, the target distribution $\pi$ is the solution to the optimization problem defined on the set $\cP_2(\X)$ of probability measures $\mu$ such that $\int \|x\|^2 d\mu(x) < \infty$ by:
\begin{equation}\label{eq:minKL-1}
\min_{\mu \in\cP_2(\X)} \KL(\mu|\pi),
\end{equation}
where $\KL$ denotes the Kullback-Leibler divergence, and  assuming $\pi \in \cP_2(\X)$.  Many existing methods for the sampling task can be related to this optimization problem. Variants of the Langevin Monte Carlo algorithm \citep{durmus2016sampling,dalalyan2019user} can be seen as time-discretized schemes of the gradient flow of the relative entropy. These methods generate a Markov chain whose law converges to $\pi$ under mild assumptions, but the rates of convergence deteriorate quickly in high dimensions \citep{durmus2018analysis}. Variational inference methods instead restrict the search space of problem \eqref{eq:minKL-1} to a family of parametric distributions \citep{zhang2018advances,ranganath2014black}. These methods are much more tractable in the large scale setting, since they benefit from efficient optimization methods (parallelization, stochastic optimization); however they can only return an approximation of the target distribution.

Recently, the Stein Variational Gradient Descent (SVGD) algorithm \citep{liu2016stein} was introduced as a non-parametric alternative to variational inference methods. It uses a set of interacting particles to approximate the target distribution, and applies iteratively to these particles a form of gradient descent of the relative entropy, where the descent direction is restricted to belong to a unit ball in a Reproducing Kernel Hilbert space (RKHS) \citep{steinwart2008support}. In particular, this algorithm can be seen as a discretization of the gradient flow of the relative entropy on the space of probability distributions, equipped with a distance that depends on the kernel \citep{liu2017stein,duncan2019geometry}. 
The empirical performance of this algorithm and its variants have been largely demonstrated in various tasks in machine learning such as Bayesian inference \citep{liu2016stein,feng2017learning,liu2018riemannian,detommaso2018stein}, learning deep probabilistic models \citep{wang2016learning,pu2017vae}, or reinforcement learning \citep{liu2017policy}. In the limit of infinite  particles, the algorithm is known to converge  to the target distribution under appropriate growth assumptions on the potential  \citep{lu2019scaling}. Nonetheless, its non-asymptotic analysis remains incomplete:  in particular, to the best of our knowledge, quantitative rates of convergence have yet to be obtained. The present paper aims at answering this question. Our first contribution is to provide in the infinite-particle regime a descent lemma showing that SVGD decreases at each iteration for a sufficiently small but constant step-size, with an analysis different from \cite{liu2017stein}. We view this problem as an optimization problem over $\cP_2(\X)$ equipped with the Wasserstein distance, and use this  framework and optimization techniques to obtain our results. Our second contribution is to provide in the finite particle regime, a propagation of chaos  bound that quantifies the deviation of the empirical distribution of the particles to its population version.

%Our analysis and conditions on the step-size differs from the descent lemma in \cite{liu2017stein}. 
%This contrasts with the descent lemma in \cite{liu2017stein} that holds for a step-size depending on $\mu_n$.
 %Then, by leveraging a kernel version of the log-Sobolev inequality, called the Stein log-Sobolev inequality as proposed by \cite{duncan2019geometry}, we derive rates of convergence in terms of the KL objective.
 % For the finite number of particles regime, we provide a bound that quantifies the deviation of the empirical distribution of the particles to its population version.

This paper is organized as follows. \Cref{sec:preliminaries} introduces the background needed on optimal transport, while \Cref{sec:pres_svgd} presents the point of view adopted to study SVGD  in the infinite number of particles regime and reviews related work.
 \Cref{sec:continuous_time} studies the continuous time dynamics of SVGD. Our main result is presented in \Cref{sec:convergence}, where we provide a descent lemma and rates of convergence for the SVGD algorithm.
We also provide a convergence result of the finite particle system to its population version \Cref{sec:finite_particles}. % at each iteration.%\aknote{adil deferred this result to the appendix}. 
%Finally we discuss our results in \Cref{sec:conclusion} and the potential impact of our paper in \Cref{sec:broader_impact}. 
The complete proofs and toy experiments are deferred to the appendix.

%\vspace{-2mm}
\section{Preliminaries on optimal transport}\label{sec:preliminaries}
%\vspace{-2mm}
Let $\X=\R^d$. We denote by $C^l(\X)$ the space of $l$ continuously differentiable functions on $\X$. If $\psi : \X \to \R^p$, $p \geq 0$, is differentiable, we denote by $J \psi : \X \to \R^{p \times d}$ the Jacobian matrix of $\psi$. If $p = 1$, the gradient of $\psi$ denoted $\nabla \psi$ is seen as a column vector. Moreover, if $\nabla \psi$ is differentiable, the Jacobian of $\nabla \psi$ is the Hessian of $\psi$ denoted $H_{\psi}$. If $p = d$, $div(\psi)$ denotes the divergence of $\psi$, \textit{i.e.}, the trace of the Jacobian. The Hilbert-Schmidt norm of a matrix is denoted $\|\cdot\|_{HS}$ and the operator norm denoted $\|\cdot\|_{op}$.
\vspace{-2mm}
\subsection{The Wasserstein space and the continuity equation}
\vspace{-2mm}
In this section, we recall some background from optimal transport. The reader may refer to~\cite{ambrosio2008gradient} for more details.

Consider the set $\cP_2(\X)$ of probability measures $\mu$ on $\X$ with finite second order moment. For any $\mu \in \cP_2(\X)$, $L^2(\mu)$ is the space of functions $f : \X \to \X$ such that $\int \|f\|^2 d\mu < \infty$. %Note that the identity map $I$ is an element of $L^2(\mu)$. 
If $\mu \in \cP_2(\X)$, we denote by $\Vert \cdot \Vert_{L^2(\mu)}$ and $\ps{\cdot,\cdot}_{L^2(\mu)}$ respectively the norm and the inner product of the Hilbert space $L^2(\mu)$. 
Given a measurable map $T:\X\to \X$ and $\mu\in \cP_2(\X)$, we denote by $T_{\#}\mu$ the pushforward
measure of $\mu$ by $T$, characterized by the transfer lemma $\int \phi(T(x))d\mu(x)=\int \phi(y)dT_{\#}\mu(y)$, 
for any measurable and bounded function $\phi$.
Consider $\mu,\nu \in \cP_2(\X)$, the 2-nd order Wasserstein distance is defined by
$W_2^2 (\mu, \nu) = \inf_{s \in \mathcal{S}(\mu,\nu)} \int \|x-y\|^2 ds(x,y)$, 
where $\mathcal{S}(\mu,\nu)$ is the set of couplings between $\mu$ and $\nu$, \textit{i.e.} the set of nonnegative measures $s$ over $\X \times \X$ such that $P_{\#} s = \mu$ (resp. $Q_{\#} s = \nu$) where $P : (x,y) \mapsto x$ (resp. $Q : (x,y) \mapsto y$) denotes the projection onto the first (resp. the second) component. The Wasserstein distance is a distance over $\cP_2(\X)$. The metric space $(\cP_2(\X),W_2)$ is called the Wasserstein space. 

%\subsection{The continuity equation}
Let $T > 0$. Consider a weakly continuous map $\mu : (0,T) \to \cP_2(\X)$. The family $(\mu_t)_{t \in (0,T)}$ satisfies a continuity equation if there exists $(v_t)_{t \in (0,T)}$ such that $v_t \in L^2(\mu_t)$ and 
\begin{equation}
    \frac{\partial \mu_t}{\partial t} + div(\mu_t v_t) = 0
\end{equation}
holds in the distributional sense. A family $(\mu_t)_t$ satisfying a continuity equation with $\|v_t\|_{L^2(\mu_t)}$ integrable over $(0,T)$ is said absolutely continuous. Among the possible processes $(v_t)_t$, one has a minimal $L^2(\mu_t)$ norm and is called the velocity field of $(\mu_t)_t$. In a Riemannian interpretation of the Wasserstein space \citep{otto2001geometry}, this minimality condition can be characterized by $v_t$ belonging to the tangent space to $\cP_2(\X)$ at $\mu_t$ denoted $T_{\mu_t}\cP_2(\X)$, which is a subset of $L^2(\mu_t)$. 
\vspace{-2mm}
\subsection{A functional defined over the Wasserstein space}
\vspace{-2mm}
Consider $\pi \propto \exp(-V)$ where $V:\X \to \R$ is a smooth function, i.e. $V$ is $C^2(\X)$ and its Hessian $H_V$ is bounded from above.
For any $\mu, \pi \in \cP_2(\X)$, the Kullback-Leibler divergence of $\mu$ w.r.t. $\pi$ is defined by
\begin{equation*}
    \KL(\mu|\pi) = \int \log\left(\frac{d\mu}{d\pi}(x)\right)d\mu(x)
\end{equation*}
if $\mu$ is absolutely continuous w.r.t. $\pi$ with Radon-Nikodym density $d\mu/d\pi$, and $\KL(\mu|\pi) = +\infty$ otherwise. Consider the functional
$\KL(.|\pi) : \cP_2(\X) \to [0,+\infty)$, $\mu \mapsto \KL(\mu|\pi)$  %$\cF : \cP_2(\X) \to (-\infty,+\infty]$ 
defined over the Wasserstein space. We shall perform differential calculus over this space for such a functional, which is a "powerful way of computing" \citep[Section 8.2]{villani2003topics}.
% by $\cF(\mu)=KL(\mu|\pi)$. 
If $\mu \in \cP_2(\X)$ satisfies some mild regularity conditions, the (Wasserstein) gradient of $\KL(.|\pi)$ at $\mu$ is denoted by $\nabla_{W_2} \KL(\mu|\pi) \in L^2(\mu)$ and defined by $\nabla \log\left(\frac{d\mu}{d\pi}\right)$. Moreover, the (Wasserstein) Hessian of $\KL(.|\mu)$ at $\mu$ is an operator over $T_\mu \cP_2(\X)$ defined by 
\begin{equation}
    \ps{v,H_{\KL(.|\pi)}(\mu) v}_{L^2(\mu)} = \E_{X \sim \mu} \left[\ps{v(X), H_V(X) v(X)}+\|J v(X)\|_{HS}^2\right]
\end{equation}
for any tangent vector $v \in T_\mu \cP_2(\X)$. Note that the Hessian of $\KL(.|\pi)$ is not bounded from above.  %(because of the second term due to the negative entropy)
An important property of the Wasserstein gradient is that it satisfies a chain rule. Let $(\mu_t)_t$ be an absolutely continuous curve s. t. $\mu_t$ has a density. Denote $(v_t)$ the velocity field of $(\mu_t)$. If $\varphi(t) = \KL(\mu_t|\pi)$, then under mild technical assumptions  $\varphi'(t) = \ps{v_t,\nabla_{W_2}\KL(\mu_t|\pi)}_{L^2(\mu_t)}$ (see~\cite{ambrosio2008gradient}).

% of a differentiable map $\phi : \R^d \to \R$. If $\psi : \R^d \to \R^d$ is differentiable, $\nabla \cdot \psi$ denotes the divergence of $\psi$. 

% and The differential operator
% $\nabla$ is viewed as a column vector of the same size as $x\in \R^d$. Hence, when $\phi:\R^d\to \R$ , then $\nabla \phi$ is a $\R^d$-valued
% function and $\nabla \phi(x)=( \sum_{i=1}^{d}\partial_{x_i} \phi(x))_{i=1}^d$. The divergence operator $div$ is applied on a continuously differentiable vector field $F:\X \to \R^d$, $F=\sum_{i=1}^d F_i$ as $div(F)=\nabla. F= \sum_{i=1}^d \frac{\partial_i F}{\partial x_i}$.

%\vspace{-2mm}
\section{Presentation of Stein Variational Gradient Descent (SVGD)}
\label{sec:pres_svgd}
%\vspace{-2mm}
In this section, we present our point of view on SVGD in the infinite number of particles regime.

\vspace{-2mm}
\subsection{Kernel integral operator}
\vspace{-2mm}
Consider a positive semi-definite kernel $k : \X \times \X \to \R$ and $\cH_0$ its corresponding RKHS of real-valued functions on $\X$. 
The space $\mathcal{H}_0$ is a Hilbert space with inner product $\ps{\cdot,\cdot}_{\cH_0}$ and norm $\Vert \cdot \Vert_{\cH_0}$ (see \cite{smola1998learning}). 
Moreover, $k$ satisfies the reproducing property: $
\forall \; f \in \cH_0,\; f(x)=\ps{f,k(x,.)}_{\cH_0}$.  Denote by $\cH$ the product RKHS consisting of elements $f=(f_1, \dots, f_d)$ with $f_i\in \cH_0$, and with a standard inner product $\ps{f,g}_{\cH}=\sum_{i=1}^d\ps{f_i,g_i}_{\cH_0}$.
Let $\mu \in \cP_2(\X)$; 
%AG: I removed this
%In some circumstances this is equivalent for the kernel to be characteristic \cite{sriperumbudur2011universality}. 
the integral operator associated to kernel $k$ and measure $\mu$ denoted by
$S_{\mu} : L^2(\mu) \rightarrow \cH$ is 
\begin{equation}\label{eq:integral_operator}
S_{\mu}f  = \int k(x,\cdot)f(x)d\mu(x).
\end{equation}
%AG: I fixed the below.
We make the key assumption that $\int k(x,x) d\mu(x)<\infty$ for any $\mu\in \cP_2(\X)$; which implies that $\cH \subset L^2(\mu)$. Consider functions $f,g\in L_2(\mu)\times\cH$ and denote the inclusion $\iota:\cH\to L^2(\mu),$ with $\iota^*=S_{\mu}$ its adjoint.
Then following e.g. \cite[Chapter 4]{steinwart2008support}, we have
\begin{equation}\label{eq:int_operator_ps}
\hspace{-0.3cm}  \ps{ f,\iota g}_{L^2(\mu)}= \ps{\iota^*f, g}_{\cH}=  \ps{S_{\mu}f,g}_{\cH}.
\end{equation}
When the kernel is integrally strictly positive definite, then  $\cH$  is dense in  $L^2(\mu)$ for any probability measure $\mu$ \citep{sriperumbudur2011universality}. We also define $P_{\mu}:L^2(\mu)\to L^2(\mu)$ the operator $P_{\mu}=\iota S_{\mu}$; notice that it differs from $S_{\mu}$ only in its range.
% For notation purposes we will omit the injection $\iota$ when taking the $L^2$ scalar product with an element of $\cH$.
\vspace{-2mm}
\subsection{Stein Variational Gradient Descent}
\vspace{-2mm}
We can now present the Stein Variational Gradient Descent (SVGD) algorithm~\citep{liu2016stein}. The goal of this algorithm is to provide samples from a target distribution $\pi \propto \exp(-V)$ with positive density w.r.t. Lebesgue measure and known up to a normalization constant. Several point of views on SVGD have been adopted in the literature. In this paper, we view SVGD as an optimization algorithm~\citep{liu2017stein} to minimize the Kullback-Leibler (KL) divergence w.r.t. $\pi$, see Problem~\eqref{eq:minKL-1}.  
Denote $\KL(.|\pi) : \cP_2(\X) \to [0,+\infty)$ the functional $\mu \mapsto \KL(\mu|\pi)$. % Our approach relies on the simple fact $\pi = \argmin KL(\mu|\pi)$. 
More precisely, in order to obtain samples from $\pi$, SVGD applies a gradient descent-like algorithm  to the functional $\KL(.|\pi)$. The standard gradient descent algorithm in the Wasserstein space applied to $\KL(.|\pi)$, at each iteration $n\ge 0,$ is
\begin{equation}
    \label{eq:GD}
    \mu_{n+1} = \left(I -\gamma \nabla \log\left(\frac{\mu_n}{\pi} \right)\right)_{\#} \mu_n,
\end{equation}
where $\gamma>0$ is a step size and $I$ the identity map. This corresponds to a forward Euler discretization of the gradient flow of $\KL(.|\pi)$ ~\citep{wibisono2018sampling}, and can be seen as a Riemannian gradient descent where the exponential map at $\mu$ is the map $\phi \mapsto (I+\phi)_{\#}\mu$ defined on $L^2(\mu)$.
%\footnote{Algorithm~\eqref{eq:GD} can be seen as a Riemannian gradient descent in the Wasserstein space, see~\cite{wibisono2018sampling}}. 
Therefore, the gradient descent algorithm would require to estimate the density of $\mu_n$ based on samples, which can be demanding (though see \Cref{rk:langevin} below). We next examine the analogous SVGD iteration, 
\begin{equation}\label{eq:svgd_update}
    \mu_{n+1}=\left(I-\gamma P_{\mu_n}\nabla \log\left(\frac{\mu_n}{\pi}\right)\right)_{\#}\mu_n.
\end{equation}
Instead of using $\nabla_{W_2}\KL(\mu_n|\pi)$ as the gradient, SVGD uses $P_{\mu_n} \nabla_{W_2} \KL(\mu_n|\pi).$ This can be seen as the gradient of $\KL(.|\pi)$ under the inner product of $\cH,$ since $
%\begin{equation}
    \ps{S_{\mu} \nabla_{W_2} \KL(\mu|\pi),v}_{\cH} = \ps{\nabla_{W_2} \KL(\mu|\pi), \iota v}_{L^2(\mu)}
%\end{equation}
$
for any $v \in \cH$. The important fact is that given samples of $\mu$, the evaluation of $P_{\mu}\nabla_{W_2} \KL(\mu|\pi)$ is simple. Indeed if $\lim_{\|x\|\to \infty} k(x,.)\pi(x)=0$,
%under appropriate conditions on $k$ and $\pi$\footnote{$\pi(x)k(x,.)=0\; \forall x \in \partial \X$ when $\X$ is compact, or $\lim_{\|x\|\to \infty} k(x,.)\pi(x)=0$ when $\X=\R^d$}(see~\cite{liu2017stein}), %\manote{This condition is not enough one needs an assumption about the decay of this limit, something like  $O(1/x^{d-1})$  }
\begin{equation}\label{eq:svgd_vector_field}
    P_{\mu}\nabla \log\left(\frac{\mu}{\pi}\right)(\cdot)=-\int[ \nabla \log\pi(x)k(x,\cdot)+\nabla_x k(x,\cdot)]d\mu(x),
\end{equation}
using an integration by parts (see~\cite{liu2017stein}).
\begin{remark}\label{rk:langevin}
An alternative sampling algorithm which does not imply to compute the exact gradient of the KL is the Unadjusted Langevin Algorithm (ULA). It is an implementable algorithm that computes a gradient step with $\nabla \log \pi$, and a flow step adding a Gaussian noise to the particles. However, it is not a gradient descent discretization; it rather corresponds to performing a Forward-Flow (FFl) discretization, which  is biased \cite[Section 2.2.2]{wibisono2018sampling}. 
\end{remark}
\vspace{-2mm}
\subsection{Stein Fisher information}\label{sec:KSD}
\vspace{-2mm}
The squared RKHS norm of the gradient $S_{\mu} \nabla \log(\frac{\mu}{\pi})$ is defined as the Stein Fisher Information:
\begin{definition}
	Let $\mu \in \cP_2(\X)$. The \textit{Stein Fisher Information} of $\mu$ relative to $\pi$ \cite{duncan2019geometry} is defined by  :
	\begin{equation}\label{eq:I_stein}
	I_{Stein}(\mu|\pi)= \|S_{\mu}\nabla \log\left(\frac{\mu}{\pi}\right)\|^2_{\cH}.
	\end{equation}
\end{definition}
\begin{remark}
	Notice that since $P_{\mu}=\iota S_{\mu}$ with $\iota^{*}=S_{\mu}$, we can write $I_{Stein}(\mu|\pi)=\ps{\nabla \log(\frac{\mu}{\pi}), P_{\mu}\nabla \log(\frac{\mu}{\pi})}_{L^2(\mu)}$.
\end{remark}
%\manote{More context about $I_stein$ say that it is also the stein discrepancy that is used for testing.}
In the literature the quantity \eqref{eq:I_stein} is also referred to as the squared Kernel Stein Discrepancy (KSD), used in nonparametric statistical tests for goodness-of-fit \citep{liu2016kernelized, chwialkowski2016kernel, gorham2017measuring}. The KSD provides a discrepancy between probability distributions, which depends on $\pi$ only through the score function $\nabla \log \pi$ that can be calculated without knowing the normalization constant of $\pi$. 
Whether the convergence of the KSD to zero, i.e. $I_{stein}(\mu_n|\pi) \to 0$ when $n\to \infty$ implies the weak convergence of $(\mu_n)$ to $\pi$ (denoted $\mu_n \to \pi$) depends on the choice of the kernel relatively to the target. This question has been treated in~\cite{gorham2017measuring}. Sufficient conditions include $\pi$ being distantly dissipative \footnote{i.e. such that $ \lim \inf_{r \to \infty}   \kappa(r) > 0$ for $\kappa(r) =
	\inf \{ -2 \ps{ \nabla \log \pi (x)- \nabla \log \pi(y),x-y }/\|x-y\|_2^2 ; \|x-y\|_2=r
	\}$. This includes finite Gaussian
	mixtures with common covariance and all distributions
	strongly log-concave outside of a compact set, including
	Bayesian linear, logistic, and Huber regression posteriors
	with Gaussian priors.} which is similar to strong log concavity outside a bounded domain, and the kernel having a slow decay rate (e.g. being translation invariant with a non-vanishing Fourier transform, or $k$ being the inverse multi-quadratic kernel defined by $k(x,y)=(c^2 + \|x- y\|_2
^2)^{\beta}$ for $c >0$
and $\beta \in [-1, 0]$). 
In these cases, $I_{stein}(\mu_n|\pi) \to 0$ implies $\mu_n \to \pi$.

In order to study the continuous time dynamics of SVGD,~\cite{duncan2019geometry} introduced a kernel version of a log-Sobolev inequality (which usually upper bounds the KL by the Fisher divergence \citep{vempala2019rapid}).
 %\aknote{I'll put log-Sobolev and langevin dynamics in the appendix later}, 
\begin{definition}
	We say that $\pi$ satisfies the \textit{Stein log-Sobolev inequality} with constant $\lambda >0$ if:
	\begin{equation}\label{eq:stein_log_sobolev}
	\KL(\mu|\pi)\le \frac{1}{2\lambda} I_{Stein}(\mu|\pi).
	\end{equation}
\end{definition}
The functional inequality \eqref{eq:stein_log_sobolev} is not as well known and understood as the classical log-Sobolev inequality.\footnote{i.e. $\KL(\mu|\pi)\le 1/2\lambda \| \nabla \log(\frac{\mu}{\pi})\|^2_{L^2(\mu)}$, which holds for instance as soon as $V$ is $\lambda$-strongly convex.} \cite{duncan2019geometry} provided a first  investigation into when this condition might hold. They  show
 that it {\em fails} to hold if the kernel is too regular w.r.t. $\pi$, more precisely for $k\in C^{1,1}(\X \times \X)$, and if $\sum_{i=1}^d [
(\partial_iV (x))^2
k(x, x) -\partial_i V(x)(
\partial_
i^1 k(x, x) + \partial^
2_i k(x, x))
+  \partial_i^{1}
\partial_i^{2} k(x, x)]d\pi(x) < \infty$, where $\partial^1_i$ and  $\partial^2_i$
denote derivatives with respect to the first and second argument of $k$ respectively \cite[Lemma 36]{duncan2019geometry}. This holds for instance in the case where
$\pi$ has exponential tails and the derivatives of $k$ and $V$ grow at most at a polynomial rate. However, they provide interesting cases in dimension 1 where \eqref{eq:stein_log_sobolev} holds, depending on $k$ and $\pi$.  For instance, by choosing a nondifferentiable kernel that
is adapted to the tails of the target $k(x, y) = \pi(x)^{-1/2}e^{-|x-y|}\pi(y)^{-1/2}$, and if $V''(x) + (V'(x))^2/2\ge \tilde{\lambda}>0$ for any $x\in \R$, then \eqref{eq:stein_log_sobolev} holds with $\lambda=\min(1,\tilde{\lambda})$ \citep[Example 40]{duncan2019geometry}.  %Another case where it holds is for a quadratic potential $V (x) = \frac{\alpha}{2}x^2$, $\alpha > 0$ and a linear kernel $k(x, y) = xy$, then $\lambda=2\alpha \int x^2d\pi(x)$ \citep[Lemma 43]{duncan2019geometry}.  
Conditions where \eqref{eq:stein_log_sobolev} holds in higher dimensions are more challenging to establish, and are a topic of current research.

\subsection{Related work}

SVGD was originally introduced by \cite{liu2016stein},
 and was shown empirically to be competitive  with 
 state-of-the-art methods in Bayesian inference.
%AG: given the follow up work by same authors better not to say it like this.
 %, but 
%it came without theoretical foundations.
 \cite
{liu2017stein}  developed the first 
theoretical analysis and studied the weak 
convergence properties of SVGD. 
They showed that for any iteration, the empirical 
distribution of the SVGD samples (i.e., for a 
finite number of particles) weakly converges to the 
population distribution when the number of particles 
goes to infinity. In the infinite particle 
regime, they provided a descent lemma showing that 
the KL objective decreases at each iteration (see \Cref{rk:comparison_liu}).   
%The latter holds under the assumption that the step size is smaller than a threshold, which depends in a complex way on the (unknown) distribution of the particles; however a sufficient condition is . 
Finally, they 
derived the non-linear partial differential
equation (PDE)  that governs continuous time dynamics of 
SVGD, and provided a geometric intuition that 
interprets SVGD as a gradient flow of the KL divergence 
under a new Riemannian metric structure (the \textit
{Stein geometry}) induced by the kernel. \cite
{liu2018stein} studied the fixed point 
properties of the algorithm for a finite number of 
particles, and showed that it exactly estimates 
expectations under the target distribution, for a 
set of functions called the Stein matching set, 
that are determined by the Stein operator 
(depending on the target distribution) and the 
kernel. In particular, they showed that by choosing 
linear kernels, SVGD can exactly estimate the mean 
and variance of Gaussian distributions when the 
number of particles is greater than the dimension. 
They further  derived high probability bounds that 
bound the Kernel Stein Discrepancy  between the empirical 
distribution and the target measure when the kernel 
is approximated with random features. \cite
{lu2019scaling} studied the continuous 
time dynamics of SVGD in the infinite number of 
particles regime. They showed that the PDE governing 
continuous-time, infinite sample SVGD dynamics is 
well-posed, and that the law of the particle system 
(for a finite number of particles) is a weak 
solution of the equation, under appropriate growth 
conditions on the score function $\nabla \log \pi$, 
and they studied the regularity of the PDE. Finally,
 \cite{duncan2019geometry} investigated the 
contraction and equilibration properties of this 
PDE. In particular, they proposed conditions that induce  exponential 
convergence to the equilibrium in continuous time, notably as the Stein log-Sobolev inequality, which relates 
the convexity of the KL objective to the \textit
{Stein geometry} 
(see \Cref{sec:continuous_time}).
By contrast with \cite{lu2019scaling,duncan2019geometry},
we develop a
theoretical understanding of SVGD in {\em discrete time}, 
where to our knowledge rates of convergence have yet to be established.

\vspace{-2mm}
\section{Continuous-time dynamics of SVGD}\label{sec:continuous_time}
\vspace{-2mm}
This section defines and describes the SVGD dynamics in continuous time. Some of the results are already stated in \cite{liu2017stein} and \cite{duncan2019geometry} but are necessary to understand the discrete time analysis. We provide intuitive sketches of the proof ideas in the main document, which exploit the differential calculus over the Wasserstein space. Detailed proofs are given in the Appendix. 
%our proof technique is different and relies on exploiting the differential calculus over the Wasserstein space. %We sketch the proofs of the results. 

The SVGD gradient flow is defined as the flow induced by the continuity equation~\citep{liu2017stein}:
\begin{equation}\label{eq:svgd_flow}
\frac{\partial \mu_t}{\partial t}+div(\mu_t v_t ) =0, \qquad v_t := - P_{\mu_t}\nabla \log \left(\frac{\mu_t}{\pi}\right).
\end{equation}
%\asnote{We have to prove that $v_t$ is the veloc field.}
%For every $t \geq 0$, $\mu_t$ can be seen as the distribution of $x_t$, where $(x_t)$ satisfies the differential equation
%\begin{equation}\label{eq:svgd_dynamics}
%d x_t =-P_{\mu_t}\nabla \log \left(\frac{\mu_t}{\pi}\right)(x_t) dt.
%\end{equation}
Equation \eqref{eq:svgd_flow} was shown to admit a unique and well defined solution (given an initial condition $\mu_0\in \cP_2(\X)$) provided that some smoothness and growth assumptions on both kernel and target density $\pi$ are satisfied \citep{lu2019scaling}. Notice that the SVGD update \eqref{eq:svgd_update} is a forward Euler discretization of \eqref{eq:svgd_flow}. We propose to study the dissipation of the KL along the trajectory of the SVGD gradient flow. The Stein Fisher Information turns out to be the quantity that quantifies this dissipation, as stated in the next proposition.
\begin{proposition}\label{prop:dissipation} %Let $\cF:\cP(\X)\to \R$ defined by $\cF(\mu)=KL(\mu|\pi)$.	
	The dissipation of the $\KL$ along the SVGD gradient flow \eqref{eq:svgd_flow} is:
	\begin{equation}\label{eq:dissipation_KL}
	\frac{d\KL(\mu_t|\pi)}{dt}=-I_{Stein}(\mu_t|\pi).\qedhere
	\end{equation}
\end{proposition}
\begin{proof}
	Recall that $\nabla_{W_2}\KL(\mu|\pi)=\nabla \log(\frac{\mu}{\pi})$; using differential calculus in the Wasserstein space and the chain rule we have, \begin{equation*}
		%\frac{d\KL(\mu_t|\pi)}{dt}= \ps{v_t,\nabla \log\left(\frac{\mu_t}{\pi}\right)}_{L^2(\mu_t)} = \ps{v_t,P_{\mu_t}\nabla \log\left(\frac{\mu_t}{\pi}\right)}_{\cH} = -\Vert P_{\mu_t}\nabla \log\left(\frac{\mu_t}{\pi}\right) \Vert_{\cH}^2,
				\frac{d\KL(\mu_t|\pi)}{dt}= \left \langle  v_t,\nabla \log\left(\frac{\mu_t}{\pi}\right)\right \rangle_{L^2(\mu_t)} =% \left \langle v_t,P_{\mu_t}\nabla \log\left(\frac{\mu_t}{\pi}\right)\right \rangle_{\cH} =
				 -\left\| S_{\mu_t}\nabla \log\left(\frac{\mu_t}{\pi}\right)\right\|_{\cH}^2.\qedhere
	\end{equation*}
	%concluding by using $P_{\mu}=\iota S_{\mu}$ and the property \eqref{eq:int_operator_ps} to write $I_{Stein}(\mu|\pi)=\ps{\nabla \log(\frac{\mu}{\pi}), P_{\mu}\nabla \log(\frac{\mu}{\pi})}_{L^2(\mu)}$.
\end{proof}
Since $I_{Stein}(\mu|\pi)$ is nonnegative, \Cref{prop:dissipation} shows that the KL divergence with respect to $\pi$ decreases along the SVGD dynamics, i.e. the KL is a Lyapunov functional for the PDE \eqref{eq:svgd_flow}. 
%A first consequence of the dissipation is the convergence of $I_{stein}(\mu_t|\pi)$ to zero~\cite[Theorem 2.8]{lu2019scaling}. However, in the proof of~\cite[Theorem 2.8]{lu2019scaling}, the authors show that $\mu_t$ converges weakly towards $\pi$ when $V$ grows at most polynomially. However, they implictly assumed that $I_{Stein}(\mu_t|\pi)\rightarrow 0 $ which doesn't need to be true in general \cite{lesigne2010behavior}. 
It can actually be proven that $I_{Stein}(\mu_t|\pi)\rightarrow 0 $, as stated in the following proposition. Its proof is deferred to \Cref{sec:cont_time_convergence}.

\begin{proposition}\label{prop:cont_time_convergence}
	Let $\mu_t$ be a solution of \eqref{eq:svgd_flow}. Assume \Cref{ass:k_bounded}, \ref{ass:V_Lipschitz}, hold and that $\exists C>0$ such that $\int \|x\| d\mu_t(x)<C$ for all $t\ge0$.
	Then $I_{Stein}(\mu_t|\pi) \rightarrow 0$.
\end{proposition}
\begin{remark}%\aknote{to michael: i put it as a remark}
In the proof of~\citet[Theorem 2.8]{lu2019scaling}, the authors show that $\mu_t$ converges weakly towards $\pi$ when $V$ grows at most polynomially. However, they implictly assumed that $I_{Stein}(\mu_t|\pi)\rightarrow 0$ which does not need to be true in general \citep{lesigne2010behavior}. It can actually be proven that $I_{Stein}(\mu_t|\pi)\rightarrow 0 $ by controlling the oscillation of the $I_{Stein}(\mu_t|\pi)$ in time, using a semi-convexity result on the KL.
\end{remark}
%\begin{proof}	Using Proposition~\ref{prop:dissipation}, $KL(\mu_t|\pi)$ is bounded. Since $KL$ is coercive in the weak topology, $(\mu_t)$ is tight. 	Let $(\mu_{t_n})_n$ a subsequence of $(\mu_t)$ such that  $\mu_{t_n} \rightarrow \mu_{\infty}$ in the weak topology. Using Proposition~\ref{prop:dissipation}, $t \mapsto I_{stein}(\mu_t|\pi)$ is integrable over $[0,+\infty)$, \asnote{complete without saying everything}\end{proof}

A second consequence of \Cref{prop:dissipation} is the following continuous time convergence rate for the average of $I_{Stein}(\mu_t|\pi)$. It is obtained immediately by integrating \eqref{eq:dissipation_KL} and using the positivity of the KL.
\begin{proposition} For any $t\ge 0$,
	\begin{equation}
		 \min_{0 \leq s \leq t} I_{Stein}(\mu_s|\pi) \leq \frac{1}{t} \int_{0}^{t} I_{Stein}(\mu_s|\pi)ds \leq \frac{\KL(\mu_0|\pi)}{t}. 
	\end{equation}
\end{proposition}
%\begin{proof}   Using Proposition~\ref{prop:dissipation}, by integrating we have
%	\begin{equation}
%		t\min_{0 \leq s \leq t} I_{Stein}(\mu_s|\pi) \leq \int_{0}^{t} I_{Stein}(\mu_s|\pi)ds \leq KL(\mu_0|\pi) - KL(\mu_t|\pi) \leq KL(\mu_0|\pi).\qedhere
%		   \end{equation}
%\end{proof}

The convergence of $I_{Stein}(\mu_t|\pi)$ itself can be arbitrarily slow, however. To guarantee faster convergence rates of the SVGD dynamics, further properties are needed, such as convexity properties of the KL-divergence with respect to the Stein geometry. This is the purpose of the inequality~\eqref{eq:stein_log_sobolev} which implies exponential convergence of the SVGD gradient flow near equilibrium. 
Indeed, if $\pi$ satisfies the Stein log-Sobolev inequality, the Kullback-Leibler divergence converges exponentially fast along the SVGD dynamics.
\begin{proposition}\label{prop:cv_continuous}
Assume  $\pi$ satisfies the Stein log-Sobolev inequality with $\lambda>0$. Then 
\begin{equation*}
\KL(\mu_t|\pi)\le e^{-2\lambda t}\KL(\mu_0|\pi).
\end{equation*}
\end{proposition}
\begin{proof} Combining \eqref{eq:dissipation_KL} and \eqref{eq:stein_log_sobolev} yields $\frac{d\KL(\mu_t|\pi)}{dt}  \leq -2\lambda \KL(\mu_t|\pi)$.
%	\begin{equation}
%		\frac{d\KL(\mu_t|\pi)}{dt} = -I_{stein}(\mu_t|\pi) \leq -2\lambda \KL(\mu_t|\pi).
%	\end{equation}
	We conclude by applying Gronwall's lemma.
\end{proof}
In the next section, we provide a non-asymptotic analysis for SVGD. Our first result holds without any convexity assumptions on the KL, but mainly under a smoothness assumption on $\pi$, while our second result leverages \eqref{eq:stein_log_sobolev} to obtain rates of convergence.

%\vspace{-2mm}
\section{Non-asymptotic analysis for SVGD}\label{sec:convergence}
%\vspace{-2mm}

This section studies the SVGD dynamics in discrete time. Although one of the results echoes \cite{liu2017stein}[Theorem 3.3], we provide new convergence rates for the discrete time SVGD under mild conditions, and using different techniques: we return to this point in detail in \Cref{rk:comparison_liu} below. Moreover, our proof technique is different. As in the previous section, we provide intuitive sketch of the proofs exploiting the differential calculus over the Wasserstein space. Each step of the proofs is rigourously justified in the Supplementary material.

Recall that the SVGD update is defined as \eqref{eq:svgd_update}. Let	 $\mu_0\in \cP_2(\X)$	and assume that it admits a density.
%\begin{equation}	\label{eq:SVGD-analysis}\mu_{n+1} = (I - \gamma P_{\mu_{n}}\nabla\log \left(\frac{\mu_n}{\pi}\right))_{\#}\mu_n.\end{equation}
For every $n \geq 0$, $\mu_n$  is the distribution of $x_n$, where %$(x_n,\mu_n)$ satisfies
\begin{equation}\label{eq:svgd_dynamics2}
x_{n+1} = x_n - \gamma P_{\mu_n}\nabla \log \left(\frac{\mu_n}{\pi} \right)(x_n),\quad x_0\sim \mu_0.
\end{equation}
This particle update leads to the finite particles implementation of SVGD, analysed in \Cref{sec:finite_particles}.

In this section, we analyze SVGD in discrete time, in the infinite number of particles regime~\eqref{eq:svgd_update}. We propose to study the dissipation of the KL along the SVGD algorithm. The Stein Fisher Information once again  quantifies this dissipation, as in the continuous time case. Before going further, note that discrete time analyses often require more assumptions that continuous time analyses. In optimization, these assumptions typically require some smoothness of the objective function. Here, we assume the following. 
%In this section we study the behavior of the Kullback-Leibler divergence along the Stein Variational Gradient Descent algorithm in discrete time.
\newcounter{contlist}
\begin{assumplist}
	\setlength\itemsep{0.2em}
		\item \label{ass:k_bounded}
	Assume that $\exists B>0$ s.t. for all $x \in \X$,\\
	$\|k(x,.)\|_{\cH_0}\le B$ and $\|\nabla _xk(x,.)\|_{\cH}=(\sum_{i=1}^d \|\partial_{x_i}k(x_i,.)\|^2_{\cH_0})^{\frac{1}{2}}\le B$.
	\item 	\label{ass:V_Lipschitz} The Hessian  $H_V$ of $V=-\log\pi$ is well-defined and $\exists M>0$ s.t. $\|H_V\|_{op}\le M$.
	%Assume  that $\nabla \log \pi$ is $M$-Lipschitz : \\$\|\nabla \log \pi (x) - \nabla \log \pi(y)\|\le M \|x-y\|$ for any $x,y\in \X$.
	\item \label{ass:bounded_I_Stein}	Assume  that $\exists$ is $C>0$ s.t.  $I_{Stein}(\mu_n|\pi)<C$ for all $n$.

\end{assumplist}
\setcounter{contlist}{\value{enumi}}

%In this section we assume \Cref{ass:k_bounded} holds, i.e. the norms of the kernel $k(x,x)$ and $  \nabla_1.\nabla_2 k(x,x)$ are bounded by some positive constant $B^2$.  We will also rely on \Cref{ass:bounded_I_Stein}, that states that the Stein Fisher information remains bounded at all iterations by some $C>0$.

%\subsection{A descent lemma}

%The following lemma states that the boundedness of the kernel, its gradient and the Hessian of $\pi$, as well as a the moments along the trajectory, are sufficient to satisfy the boundedness of the Stein Fisher information for all $n\ge0$. 

 Under \Cref{ass:V_Lipschitz,ass:k_bounded}, a sufficient condition for~\Cref{ass:bounded_I_Stein} is $\sup_n \int \Vert x \Vert \mu_n(x)dx < \infty$. Bounded moment assumptions such as these are commonly used in stochastic optimization, for instance in some analysis of the stochastic gradient descent~\citep{moulines2011non}. Given our assumptions, we quantify the decreasing of the KL along the SVGD algorithm, also called a descent lemma in optimization.

\begin{proposition}\label{prop:descent}
	%Define $\mu_{n+1} =   (I-\gamma K \nabla f_{\pi,\mu_n}   )_{\#}\mu_n $ and 
 Assume that \Cref{ass:V_Lipschitz,ass:k_bounded,ass:bounded_I_Stein} hold. % and that $\int \Vert x \Vert \mu_n(x)dx $ remains bounded for all $n$.
		  Let $\alpha>1$ and choose $ \gamma \leq \frac{\alpha-1}{\alpha BC^{\frac{1}{2}} } $. Then:
	\begin{align}\label{eq:descent}
	\KL(\mu_{n+1}|\pi)-\KL(\mu_{n}|\pi)\leq -\gamma \left(1- \gamma \frac{(\alpha^2 + M)B^2}{2}\right)I_{stein}(\mu_n|\pi).
	\end{align}
\end{proposition}
%  We provide here a sketch of the proof and the main ideas, the reader may refer to the \Cref{sec:proof_descent}  %\Cref{sec:proof_descent} 
%  for the complete proof.
\begin{proof}
    %	Let $\cF:\cP_2(\X)\to \R$ defined by $\cF(\mu)=KL(\mu|\pi)$.
    Our goal is to prove a discrete dissipation of the form $(\KL(\mu_{n+1}|\pi)-\KL(\mu_{n}|\pi))/\gamma \leq -I_{stein}(\mu_n|\pi) + \text{error term}$. 
    %\begin{equation}
        %$\frac{KL(\mu_{n+1}|\pi)-KL(\mu_{n}|\pi)}{\gamma} \leq -I_{stein}(\mu_n|\pi) + \text{error term}$. 
   % \end{equation}    
    Our assumptions will control the error term. 
    Fix $n \ge 0$ and denote $g = P_{\mu_n}\nabla \log(\frac{\mu_n}{\pi})$, $\phi_t = I - t g$ for $t \in [0,\gamma]$ and $\rho_t = (\phi_t)_{\#}\mu_n$. Note that $\rho_0 = \mu_n$ and $\rho_{\gamma} = \mu_{n+1}$. 
    
    Under our assumptions, one can show that for any $x \in \X$, $\|g(x)\|^2\le B^2 I_{Stein}(\mu_n|\pi)$ and $\|Jg(x)\|_{HS}^2 \le B^2 I_{Stein}(\mu_n|\pi)$, using the reproducing property and Cauchy-Schwartz in $\cH$. Hence, $\|t Jg(x)\|_{op} < 1$ and $\phi_t$ is a diffeomorphism for every $t \in [0,\gamma]$. Moreover, $\|(J \phi_t)^{-1}(x)\|_{op} \leq \alpha$. Using~\cite[Theorem 5.34]{villani2003topics}, the velocity field ruling the time evolution of $\rho_t$ is $w_t \in L^2(\rho_t)$ defined by $w_t(x) = -g(\phi_t^{-1}(x))$. %Note that $w_0 = -g\in \cH$.    
    
    Denote $\varphi(t) = \KL(\rho_t|\pi)$. Using a Taylor expansion,
        $\varphi(\gamma) = \varphi(0) + \gamma \varphi'(0) + \int_{0}^{\gamma} (\gamma - t)\varphi''(t)dt$.
        We now identify each term. First, 
        \begin{equation*}
    \varphi(0) = \KL(\mu_n|\pi)\; \text{ and } \;\varphi(\gamma) = \KL(\mu_{n+1}|\pi).
        \end{equation*} 
        Then, using the chain rule~\cite[Section 8.2]{villani2003topics}, 
        \begin{equation*}
            \varphi'(t)=\ps{\nabla_{W_2} \KL(\rho_t|\pi),w_t}_{L^2(\rho_t)}\; \text{ and } \;\varphi''(t) = \ps{w_t,Hess_{\KL(.|\pi)}(\rho_t)w_t}_{L^2(\rho_t)}.
        \end{equation*}
        Therefore, $\varphi'(0) = -\ps{\nabla \log\left(\frac{\mu_n}{\pi}\right),g}_{L^2(\mu_n)}%=-\ps{g,g}_{\cH}
        =-I_{Stein}(\mu_n|\pi)$. %, using $g \in \cH$. 
        Moreover, $\varphi''(t) = \psi_1(t) + \psi_2(t)$, where
        \begin{equation*}
        \psi_1(t) = \E_{x \sim \rho_t} \left[ \ps{w_t(x), H_V(x) w_t(x)}\right] \; \text{ and } \; \psi_2(t) = \E_{x \sim \rho_t} \left[ \|J w_t(x)\|_{HS}^2 \right].
        \end{equation*}
    The first term $\psi_1(t)$ is bounded using~\Cref{ass:V_Lipschitz}, $\psi_1(t) \leq M \|g\|_{L^2(\mu_n)}^2 \leq M B^2 I_{Stein}(\mu_n|\pi)$. The second term $\psi_2(t)$ is the most challenging to bound as $\|J w\|_{HS}$ cannot be controlled by $\|w\|$ for a general $w$. However, in our case, $w_t = -g \circ (\phi_t)^{-1}$, and $-J w_t \circ \phi_t = Jg (J \phi_t)^{-1}$. Therefore, $\|J w_t\circ \phi_t(x)\|_{HS}^2 \leq \|Jg(x)\|_{HS}^2 \|(J \phi_t)^{-1}(x)\|_{op}^2 \leq \alpha^2 B^2 I_{Stein}(\mu_n|\pi)$.
    %The second term , the first term is easily bounded by
    % $\ps{v,H_V v} \leq M \|v\|^2$. Then since $\rho_t = \phi_{t\#} \mu_n$ and $v_t = -g(\phi_t^{-1})$, by the transfer lemma $\mathbb{E}_{x \sim \rho_t}[\| v_t(x)\|^2]=\mathbb{E}_{x\sim \mu_n}[\|g(x)\|^2]\le B^2I_{Stein}(\mu_n|\pi)$, hence 
    %$
    %\E_{x\sim \rho_t}[ \ps{v_t(x), H_V(x) v_t(x)}]\le M B^2 I_{Stein}(\mu_n|\pi)$.
    %  The second term is the most challenging to bound.  
    %  By the chain rule 
    %  %for any $x$ we have $-J v_t(x)=Jg(\phi_t^{-1}(x))(J\phi_t)^{-1}(\phi_t^{-1}(x))$; hence by 
    %  and the transfer lemma, we first have
    %   $\E_{x \sim \rho_t} \left[\|J v_t(x)\|_{HS}^2\right]= \E_{x \sim \mu_n} \left[\|Jg(x) (J\phi_t)^{-1}(x)\|_{HS}^2\right]$. Then,  for any $x$, $ \|Jg(x) (J\phi_t)^{-1}(x)\|_{HS}^2 \le B^2I_{Stein}(\mu_n|\pi) \| (J\phi_t)^{-1}(x)\|_{op}^2$. On the other hand, $J\phi_t=I-tJg$, and if $t<\frac{1}{B\sqrt{C}}$ then $t\|Jg(x)\|<1$ and it can be shown that $\|(I - t Jg(x))^{-1}\| \le \alpha$ for $\gamma$ chosen as in \Cref{prop:descent}. Therefore $\|Jg(x) (J\phi_t)^{-1}(x)\|_{HS}^2\le \alpha^2 B^2 I_{Stein}(\mu_n|\pi)$ and $\varphi''(t)\le (\alpha^2+M)B^2I_{Stein}(\mu_n|\pi)$. 
    Combining each of the quantity in the Taylor expansion gives the desired result.
        %Sufficiently, we shall bound the term $\|Jg)(x) J(\phi_t)^{-1}(x)\|_{HS}^2$. 
    \end{proof}
Although the Hessian of $\KL(.|\pi)$ is not bounded over the whole tangent space, our proof relies on controlling the Hessian when restricted to $\cH$. 
% The Hessian of the KL at $\mu$ for a tangent vector $v\in L^2(\mu)$ is 
% \begin{equation*}
% Hess_{\mu}\KL(\mu|\pi)= \E_{X \sim \mu} \left[  \ps{v(X), H_V(X) v(X)}+\|J v(X)\|_{HS}^2\right],
% \end{equation*}
% where $H_V$ is the Hessian of $-\log\pi$. In the general case, 
% the Hessian of the relative entropy is not bounded above due to the second term on the right hand side. However, restricting the descent direction to the RKHS $\cH$ by choosing $g=P_{\mu_n}\nabla \log(\frac{\mu_n}{\pi})$ enables to bound this term specifically using the boundedness assumption on the kernel and its gradient, and the boundnedness of $I_{Stein}(\mu_n|\pi)<C$ along the trajectory. 
% The descent result of \Cref{prop:descent} easily yields the following corollary, obtained by telescoping the inequality \eqref{eq:descent}.
Since $I_{Stein}(\mu_n|\pi)$ is nonnegative, \Cref{prop:descent} shows that the KL divergence w.r.t.  $\pi$ decreases along the SVGD algorithm, i.e. the KL is a Lyapunov functional for SVGD. A first consequence of \Cref{prop:descent} is the convergence of $I_{stein}(\mu_n|\pi)$ to zero, similarly to the continous time case, see Proposition~\ref{prop:cont_time_convergence}. Indeed, the descent lemma implies that the sequence $I_{stein}(\mu_n|\pi)$ is summable and hence converges to zero.
%The semi-convexity result in \Cref{eq:semi-convexity} ensures that this is the case by 
A second consequence of the descent lemma is the following discrete time convergence rate for the average of $I_{Stein}(\mu_n|\pi)$.

\begin{corollary}\label{cor:istein_decreases}
	Let $\alpha>1$ and $\gamma \leq \min \left(\frac{\alpha-1}{\alpha BC^{\frac{1}{2}} }, \frac{2}{(\alpha^2+M)B^2}\right)  $ and $c_{\gamma}=\gamma \left(1- \gamma \frac{(\alpha^2 + M)B^2}{2} \right)$. Then,
\begin{equation*}
\min_{k=1,\dots,n}I_{Stein}(\mu_n|\pi)\le \frac{1}{n}\sum_{k=1}^n I_{Stein}(\mu_k|\pi)\le  \frac{\KL(\mu_0|\pi)}{c_{\gamma}n}.
\end{equation*}
\end{corollary}

%However, the convergence of $I_{Stein}(\mu_n|\pi)$ itself can be arbitrarily slow. 
We illustrate the validity of the rates of \Cref{cor:istein_decreases} with simple experiments provided \Cref{sec:experiments}. 
 \Cref{cor:istein_decreases}
 provides  a $\mathcal{O}(1/n)$ convergence rate for the arithmetic mean of the Kernel Stein Discrepancy (KSD) (which metricizes weak convergence in many cases, see \cref{sec:KSD}) between the iterates $\mu_n$ and $\pi$, under \cref{ass:V_Lipschitz} to \cref{ass:bounded_I_Stein}. %(ii) an exponential convergence in terms Kullback-Leibler (KL) divergence under A1--A3 and the Stein LSI.  
 It does not rely on Stein LSI nor on convexity of $V$, unlike most of the results on Langevin Monte Carlo (LMC) which assume either (standard) LSI or convexity of $V$ \citep{vempala2019rapid,durmus2019analysis}. To guarantee convergence rates of the SVGD algorithm in terms of the KL objective, further properties are needed. We discuss the difficulty of proving rates in KL in \cref{sec:discussion_KL_rates}.
 %Similarly to the continuous time case, if $\pi$ satisfies the Stein log-Sobolev inequality, the Kullback-Leibler divergence converges exponentially fast along the SVGD algorithm.

\begin{remark}\label{rk:comparison_liu}
  A descent lemma was also obtained for SVGD in~\cite{liu2017stein}[Theorem 3.3] under a boundedness condition of the KSD and the kernel. While we obtain similar conditions on the step size, our approach, shown in the  proof sketch (and, in greater detail, the Appendix), gives clearer connections with Wasserstein gradient flows. More precisely, we prove Proposition~\ref{prop:descent} by performing differential calculus over the Wasserstein space. We are able to replace the boundedness condition on the KSD by a simple boundedness condition of the first moment of $\mu_n$ at each iteration, which echoes analyses of some optimization algorithms like Stochastic Gradient Descent \citep{moulines2011non}.
Our construction also brings with it a simple yet informative perspective, arising from the optimization literature, into
  why SVGD actually satisfies a descent lemma. In optimization, it is well known that descent lemmas can be obtained under a boundedness condition on the Hessian matrix. Here, the Hessian operator of the KL at $\mu$ is an operator on $L^2(\mu)$; and yet, {\em this operator is not bounded}~\cite[Section 3.1.1]{wibisono2018sampling}. By restricting the Hessian operator to the RKHS however, and then using the reproducing property and our assumptions, the resulting Hessian operator is provably bounded under simple conditions on the kernel and $\pi$.  %A next insight deriving from the optimization perspective is that linear rates can be obtained by combining a descent result and a Polyak-Lojasiewicz condition on the objective function \cite{karimi2016linear}.  In our case, the latter condition corresponds to the Stein log Sobolev inequality from~\cite{duncan2019geometry}. When this holds, we can easily obtain convergence rates, which were not previously known.
\end{remark}

%\begin{remark}\label{rk:langevinVSsvgd} 
	%SVGD implements a variant of a Forward discretization of the gradient flow of the KL whereas ULA implements a Forward-flow discretization. Hence, the techniques to obtain rates of convergence are very different, and we leverage techniques from the study of gradient descent in Hilbert spaces.  %Such techniques have been used to analyze ULA~\cite{durmus2019analysis}.
%\end{remark}

\section{Finite number of particles regime}\label{sec:finite_particles}

In this section, we investigate the deviation of the discrete distributions generated by the SVGD algorithm for a finite number of particles, to its population version. In practice, starting from $N$ i.i.d. samples $X_0^{i}\sim \mu_0$, SVGD algorithm updates the $N$ particles as follows :
\begin{equation}\label{eq:svgd_update_emp}
X_{n+1}^{i}=X_n^{i} - \gamma P_{\hat{\mu}_n}\nabla \log \left(\frac{\hat{\mu}_n}{\pi}\right)(X_n^{i}), \qquad \hat{\mu}_n=\frac{1}{N}\sum_{j=1}^{N}\delta_{X_n^{j}},
\end{equation}
where $\hat{\mu}_n$ denotes the empirical distribution of the interacting particles. Recall that $P_{\hat{\mu}_n}\nabla \log \left(\frac{\hat{\mu}_n}{\pi}\right)$ is well defined even if $\hat{\mu}_n$ is discrete.

%The only work tackling theoretical properties of SVGD in the finite particle regime is \cite{liu2017stein}.\aknote{complete with Liu 2018 moment matching} 

In \cite{liu2017stein}, the authors show that the empirical distribution of the SVGD samples weakly converge to its population limit for any iteration. More precisely, under the assumptions that $b(x,y)=\nabla \log \pi(x)k(x,y)+\nabla_1k(x,y)$ is jointly Lipschitz and that $\hat{\mu}_0$ converges weakly to $\mu_0$ as $N\to \infty$ (which happens by drawing $N$ i.i.d. samples of $\mu_0$), for any $n\ge0$, they show that $\hat{\mu}_n$ converges weakly to $\mu_n$. This happens as soon as \Cref{ass:k_bounded},\ref{ass:V_Lipschitz},\ref{ass:V_bounded},\ref{ass:k_lipschitz} are satisfied (since the product of bounded Lipschitz functions is a Lipschitz function):

\begin{assumplist2}

	\setlength\itemsep{1em}
	\item \label{ass:V_bounded}
	Assume that $\exists C_{V}$ s.t. for all $x \in \X$, $\|V(x)\|\le C_{V}$. 
	%	\item 	\label{ass:V_Lipschitz}
	%	Assume  that $\nabla \log \pi$ is $M$-Lipschitz : \\$\|\nabla \log \pi (x) - \nabla \log \pi(y)\|\le M \|x-y\|$ for any $x,y\in \X$.
	%	\item \label{ass:bounded_I_Stein}	Assume  that $\exists$ is $C>0$ s.t. : $I_{Stein}(\mu|\pi)<C$ for all $\mu \in \cP(\X)$.
	\item Assume that  $\exists D>0$ s.t.	\label{ass:k_lipschitz} $k$ is continuous on $\X$ and $D$-Lipschitz: \\$|  k(x,x') - k(y,y')| \leq D(\Vert  x-y\Vert + \Vert x'-y' \Vert ) $ for all $x,x',y,y' \in \X$,\\
	
	and $k$ is continuously differentiable on $\X$ with $D$-Lipschitz gradient: \\$\Vert  \nabla k(x,x') - \nabla k(y,y')\Vert \leq D(\Vert  x-y\Vert + \Vert x'-y' \Vert ) $ for all $x,x',y,y' \in \X$.
\end{assumplist2}

Under these assumptions, we quantify the dependency on the number of particles in the following proposition.

\begin{proposition}\label{prop:finite_particle}
	Let $n\ge 0$ and $T>0$. Let $\mu_n$ and $\hat{\mu}_n$ be defined by \eqref{eq:svgd_update} and \eqref{eq:svgd_update_emp} respectively. Under \Cref{ass:k_bounded},\ref{ass:V_Lipschitz},\ref{ass:V_bounded},\ref{ass:k_lipschitz} for any $0\le n \le \frac{T}{\gamma}$:
	\begin{equation*}
	\E[W_2^2(\mu_n,\hat{\mu}_n)]\le \frac{1}{2}\left(\frac{1}{\sqrt{N}} \sqrt{var(\mu_0)}e^{LT} \right)(e^{2LT}-1)
	\end{equation*}
	where $L$ is a constant depending on $k$ and $\pi$.
\end{proposition}

\Cref{prop:finite_particle}, whose proof is provided \cref{sec:proof_particles}, uses techniques from \cite{jourdain2007nonlinear}. It is a non-asymptotic result in the sense that it provides an explicit bound. However, it is not a bound that helps quantify the rate of minimization of the objective function, but a bound between the population distribution $\mu_n$ and its particle approximation $\hat{\mu}_n$. Such results are referred to \textit{propagation of chaos}
in the PDE literature, where having the constant
$C$ %=\mathcal{O}(\exp(T)) 
depending on $T$ 
is common. Getting a similar bound with $C$ not depending on $T$ would be a much stronger result referred to as \textit{uniform in time propagation of chaos}. 
Such results, which are subject to active research in PDE, are hard to obtain. 
Among the recent exceptions is \cite{durmus2018elementary}  who consider the process $dx_t=-\nabla U(x_t) - \nabla W * \mu_t(x_t)dt$ and manage to prove such results when $U$ is strictly convex outside of a ball. However in SVGD (see \eqref{eq:svgd_vector_field}), the attractive force $ \nabla \log \pi(x)k(x,.)$ cannot be written as the gradient of a confinement potential $U:\R^d \to \R$ in general. Hence these results do not apply, and the convergence rate for SVGD using 
$\hat{\mu}_n$ 
remains an open problem.

%The time growing constant is common in this interacting particle system literature; obtaining uniform in time bounds for the propagation of chaos require substantial further work and assumptions on the potential \cite{durmus2018elementary}.

%\vspace{-2mm}
\section{Conclusion}\label{sec:conclusion}
%\vspace{-2mm}

In this paper, we provide a non-asymptotic analysis for the SVGD algorithm. Our results build upon the connection of SVGD with gradient descent in the Wasserstein space \citep{liu2017stein}.
In establishing these results, we draw on perspectives and techniques used to establish convergence in optimization. 
%and on recent results on the Stein geometry induced by the kernel characterizing the convexity of the objective (depending on the target $\pi$) with respect to this geometry. %Our results (eg the use of the Stein log Sobolev inequality) advocates for the use of kernels that are not too regular relatively to the target distribution $\pi$.
Several questions remain open. Firstly, the question of deriving rates of convergence of SVGD (in the infinite particle regime) in terms of the Kullback-Leibler objective, when the potential $V$ is convex or when $\pi$ satisfies some log Sobolev inequality. Secondly, the question of deriving a unified bound for the convergence of $\hat{\mu}_n$ to $\pi$ (decreasing as the number of iterations $n$ and number of particles $N$ go to infinity). This would require to obtain a uniform in time propagation of chaos result for the SVGD particle system. 
Finally, another further direction would be to study SVGD dynamics when the kernel depends on the current distribution. These kind of dynamics arise in black-box variational inference and Generative Adversarial Networks \citep{chu2020equivalence} (in which case the kernel is the neural tangent kernel introduced by \cite{jacot2018neural}).

%\vspace{-2mm}
\section{Broader impact}\label{sec:broader_impact}
%\vspace{-2mm}
This paper aims at bringing more theoretical understanding to the Stein Variational Gradient Descent algorithm. This algorithm is widely used by machine learning practitioners but its non asymptotic properties are not as well-known as the ones of the Langevin Monte Carlo algorithm which can be considered as its competitor. %We believe our analysis would be of interest for the Neurips community, since a large part of work related to SVGD was published at this conference \cite{liu2016stein,liu2017stein,detommaso2018stein,liu2018stein}.

\section{Funding disclosure}

AK, MA, and AG thank the Gatsby Charitable Foundation for the financial support. 

%\section*{References}
%\medskip

\small
\bibliographystyle{apalike}
\bibliography{biblio}

\newpage
\section{Background}

\subsection{Dissipation of the KL}
The time derivative, or the dissipation of the KL divergence along any flow is given by : 
\begin{equation}\label{eq:dissipation_KL1}
\frac{d}{dt}\KL(\mu_t|\pi)=\frac{d}{dt} \int \mu_t \log \left( \frac{\mu_t}{\pi}\right)dx = \int \frac{\partial \mu_t}{\partial t}\log \left(\frac{\mu_t}{\pi}\right) dx
\end{equation}
since the second part of the chain rule is null : 
\begin{equation*}
\int \mu_t \frac{\partial }{\partial t} \log \left( \frac{\mu_t}{\pi}\right)dx= \int \frac{\partial \mu_t}{\partial t}dx= \frac{d}{dt}\int \mu_t dx =0.
\end{equation*}
Moreover, if $\mu_t$ satisfies a continuity equation of the form :
\begin{equation*}
\frac{\partial \mu_t}{\partial t}+div(\mu_t v_t)=0
\end{equation*}
where $v_t$ is called the velocity field, then by an integration by parts :
\begin{align*}
\frac{d}{dt}\KL(\mu_t|\pi)&= -\int div (\mu_t(x) v_t(x))\log \left(\frac{\mu_t}{\pi}\right)dx\\
&=  \int v_t(x) \nabla \log \left(\frac{\mu_t}{\pi}\right)(x) \mu_t(x)dx =\ps{v_t, \nabla \log \left(\frac{\mu_t}{\pi}\right)}_{L^2(\mu_t)}.\numberthis \label{eq:dissipation_KL2}
\end{align*}

\subsection{Descent lemma for Gradient Descent in $\R^d$}\label{sec:descent_lemma_Rd}

In this section we show how to obtain a descent lemma for the gradient descent algorithm. We do not claim any generality here, the goal of this section is to provide an intuition behind the proof of \Cref{prop:descent} for SVGD. 

Consider $F : \R^d \to \R$ a $C^2(\R^d)$ function with Hessian $H_F$, and the gradient descent algorithm written at iteration $n+1$:
\begin{equation}
    x_{n+1} = x_n -\gamma \nabla F(x_n).
\end{equation}

Consider $n \geq 0$ fixed. For every $t \geq 0$, denote $x(t) = x_n - t\nabla F(x_n)$. Then, $x(0) = x_n$ and $x(\gamma) = x_{n+1}$. We assume that there exists $M \geq 0$ %(possibly depending on $n$) 
such that for every $t \geq 0$, $\|H_F(x(t))\| \leq M$.

Denote $\varphi(t) = F(x(t))$. Using Taylor expansion,
\begin{equation}
    \label{eq:Taylor-reste-integral}
    \varphi(\gamma) = \varphi(0) + \gamma \varphi'(0) + \int_{0}^{\gamma} (\gamma - t)\varphi''(t)dt.
\end{equation}
Denote by $\dot{x}$ the derivative of $x$. We now identify each term. First, $\varphi(0) = F(x_n)$ and $\varphi(\gamma) = F(x_{n+1})$. Second, $\varphi'(0) = \ps{\nabla F(x(0)),\dot{x}(0)} = \ps{\nabla F(x(0)),-\nabla F(x_n)} = -\|\nabla F(x_n)\|^2$. Finally, since $\ddot{x} = 0$,
\begin{equation}
    \varphi''(t) = \ps{\dot{x}(t),H_F(x(t)) \dot{x}(t)} \leq M\|\dot{x}(t)\|^2 = M\|\nabla F(x_n)\|^2.
\end{equation}
Therefore
\begin{align*}
    F(x_{n+1})& \leq F(x_n) - \gamma \|\nabla F(x_n)\|^2 + M\int_{0}^{\gamma} (\gamma - t) \|\nabla F(x_n)\|^2 dt\\
    & \leq F(x_n) - \gamma \|\nabla F(x_n)\|^2 + \frac{M\gamma^2}{2}\|\nabla F(x_n)\|^2.    \numberthis \label{eq:Taylor-Lagrange}
\end{align*}

\section{Proofs}

\subsection{Proof of \Cref{prop:cont_time_convergence}}\label{sec:cont_time_convergence}

% Define the Hessian operator as in \cite[Lemma 22]{duncan2019geometry}:
% \begin{align}
% Hess_{\rho}(\Psi, \Psi) = \sum_{i,j=1}^{d} \mathbb{E}_{x,x'\sim \rho}\left[ \partial_i \Psi(x)\partial_j(x')q_{i,j}[\rho](y,z)\right]. 
% \end{align}
% with $q_{i,j}[\rho_t](y,z)$ given in \cite[Lemma 22]{duncan2019geometry}. We have the following proposition: 
\begin{proposition}\label{eq:semi-convexity}
	Under  \Cref{ass:k_bounded}, \ref{ass:V_Lipschitz}, and assuming $\exists C>0$ such that $\int \|x\| d\mu_t(x)<C$ for all $t\ge0$,  there exists $\lambda \in \mathbb{R}^+ $ such that:
	\begin{align}
 	\left\vert\frac{d I_{Stein}(\mu_t|\pi) }{dt} \right\vert \leq\lambda I_{Stein}(\mu_t|\pi).
 \end{align}	
\end{proposition}
\begin{proof}
	We first need to compute $D_t = \frac{d I_{Stein}(\mu_t|\pi) }{dt}$. We denote by $ v_t = S_{\mu_t} \nabla \log(\frac{\mu_t}{\pi})$. % the vector field. 
	Recalling that $  I_{Stein}(\mu_t|\pi) = \sum_{i=1}^d \Vert  v_t^{i} \Vert^2_{\mathcal{H}_0}$ we have by differentiation that:
	\begin{align}
		D_t = 2\sum_{i=1}^d\langle  v_t^{i},\frac{d v_t^{i}}{dt} \rangle_{\mathcal{H}_0}
	\end{align}
	We thus need to compute each component  $\frac{d v_t^{i}}{dt}$. Those are given by direct calculation: 
	\begin{align*}
		\frac{d v_t^{i}}{dt}(x)=&\int[\partial_{i}\log \pi(x')k(x',x)+\partial_{i}k(x',x)]\frac{d\mu_{t}(x')}{dt}dx'\\=&-\int\ps{\nabla[\partial_{i}\log \pi(x')k(x',x)+\partial_{i}k(x',x)],v_{t}(x')}d\mu_{t}(x')\\=&-\int\sum_{i,j}\left[\partial_{i}\partial_{j}\log \pi(x')k(x',x)+\partial_{i}\log \pi(x')\partial_{j}k(x',x)+\partial_{j}\partial_{i}k(x',x)\right] v_{t}^{j}(x')d\mu_{t}(x').
	\end{align*}
	where the second line uses an integration by parts. Hence by using the reproducing property,
	\begin{align*}
	D_t &=	2\int\sum_{i,j}\left[\partial_{i}\partial_{j}\log \pi(x') v_t^{i}(x')+\partial_{i}\log \pi(x')\partial_{j} v_t^{i}(x')+\partial_{j}\partial_{i} v_t^{i}(x')\right] v^{j}_{t}(x')d\mu_{t}(x')
\end{align*}
We will use the reproducing property recalling that each component $ v_t^{i}$ is an element of the RKHS $\mathcal{H}_0$, i.e: $ v_t^{i}(x) = \langle  v_t^{i},k(x,.) \rangle_{\mathcal{H}_0} $, hence:
\begin{align}
	D_t = 2\sum_{i,j} \langle v_t,  A_{i,j}  v_t  \rangle_{\mathcal{H}_0},
\end{align}
 where $A_{i,j}$ are operators given by:
 \begin{align*}
 	A_{i,j} =&  \int  k(x',.)\otimes k(x,.)\partial_{i}\partial_j\log \pi(x')d\mu_t(x)d\mu_t(x')\\ 
 	&+\int  \partial_i k(x'.)\otimes k(x,.)\partial_i \log \pi(x')d\mu_t(x)d\mu_t(x') \\
 	&+ \int \partial_i k(x',.)\otimes \partial_j k(x,.) d\mu_t(x)d\mu_t(x').
 \end{align*}
 We need to show that the $A_{i,j}$ have a bounded Hilbert-Schmidt norm at all times $t$. Indeed, if $\Vert A_{i,j} \Vert_{HS} \leq R$ for some $R>0$, then we directly conclude that:
 \begin{align}
 	\vert D_t \vert \leq d R \sum_{i=1}^d \Vert  v_t^{i}\Vert_{\mathcal{H}_0}^2 = dR I_{Stein}(\mu_t|\pi).
 \end{align} 
 By assumptions on the kernel and Hessian of $\log \pi$ we have that:
 \begin{align*}
 	\Vert A_{i,j} \Vert_{HS}\leq & \int \Vert k(x',.) \Vert_{\mathcal{H}_0} \vert \partial_i\partial_j \log\pi(x') \vert d\mu_t(x') \int\Vert k(x,.) \Vert_{\mathcal{H}_0} d\mu_t(x)\\
 	&+ \int \Vert \partial_i k(x',.) \Vert_{\mathcal{H}_0}\vert \partial_i \log\pi(x') \vert d\mu_t(x') \int\Vert k(x,.) \Vert_{\mathcal{H}_0} d\mu_t(x)\\
 	& + \left(\int \Vert \partial_i k(x',.) \Vert_{\mathcal{H}_0} d\mu_t(x')\right)^2
 \end{align*}
 We recall that by assumption $\Vert k(x,.) \Vert_{\mathcal{H}_0}\leq B$,    $\Vert \partial_i k(x',.) \Vert_{\mathcal{H}_0}\leq B$ and $\Vert H_{\log \pi}(x)\Vert_{op}\leq M$. Hence, we have:
 \begin{align}
 	\Vert A_{i,j}\Vert_{HS} \leq B^2(M+1 + \int \vert \partial_i \log \pi(x) \vert  d\mu_t(x)  ).
 \end{align}
 It remains to control $ \partial_i \log \pi(x) $. This can be done under the additional assumption:
 \begin{align}
 	\int \Vert x \Vert d\mu_t(x) < C, \qquad \forall t\geq 0,
 \end{align}
 for some positive constant $C$. Hence, we have: 
 \begin{align}
 	\vert \partial_i \log \pi(x) \vert \leq \vert \partial_i \log \pi(0) \vert + M\Vert x\Vert.
 \end{align}
We finally get:
\begin{align}
	\Vert A_{i,j}\Vert_{HS} \leq B^2(M+1 +  MC +  \vert \partial_i \log \pi(0) \vert   )
\end{align}
Denoting $\lambda=d B^2(M+1 +  MC +  \vert \partial_i \log \pi(0) \vert   )$ gives the desired result.

\end{proof}

% We will further need to show that $I_{Stein}(\mu_t| \pi)$ do not oscillate too much. For that purpose, we need to compute the time derivative of $I_{Stein}(\mu_t| \pi)$. We use the expression provided in  \cite[Lemma 22]{duncan2019geometry} which states that:
% \begin{align}
% \frac{d}{dt}I_{Stein}(\mu_t|\pi) = Hess_{\mu_t}(\Psi_t, \Psi_t)
% \end{align}
% with $Hess_{\mu_t}(\Psi_t, \Psi_t)$ as defined in \Cref{eq:semi-convexity}.  According to \Cref{eq:semi-convexity} we have that:
% \begin{align}
% \vert \frac{d}{dt}I_{Stein}(\mu_t|\pi)\vert \leq \vert \lambda \vert  I_{Stein}(\mu_t|\pi).
% \end{align}
Recall, from the dissipation (Proposition~\ref{prop:dissipation}) that $\KL(\mu_t|\pi) \leq \KL(\mu_0|\pi)$. Since $\rho\mapsto  \KL(\rho|\pi)$ is weakly coercive (i.e., has compact sub-level sets in the weak topology, \cite[Theorem 20]{Erven:2014}), the family $(\mu_t)$ is weakly relatively compact. Besides, $I_{Stein}(\rho|\pi)$ is weakly continuous, therefore its supremum over the weakly relatively compact set $(\mu_t)$ is finite: $\sup_t I_{Stein}(\mu_t|\pi) < \infty$. Therefore, there exists $L \geq 0$ such that $\vert \frac{d}{dt}I_{Stein}(\mu_t|\pi)\vert \leq L$.

% On the other hand, we know that , hence, the set $C = \{ \rho|  \KL(\rho|\pi)\leq L_0 \}$ is a compact set. Since $I_{Stein}(\rho|\pi)$ is continuous under the weak topology, this implies that $I_{Stein}(\rho|\pi)$ is bounded on $C$, hence $I_{Stein}(\mu_t|\pi)$ is also bounded for all $t \geq 0$.
% Therefore any sequence $\mu_{t_k}$ of $\mu_t$ is tight and admits a subsequence that converges weakly towards some distribution $\bar{\rho}$. By continuity of $\rho \mapsto I_{Stein}(\rho|\pi) $, we can see that $I_{Stein}(\mu_{t_k}|\pi)$ is continuous on the compact set $C$ it is easy to see that $\vert \frac{d}{dt}I_{Stein}(\mu_t|\pi) \vert$ is also bounded for all $t\geq 0$. 

We can now show that $I_{Stein}(\mu_t|\pi)$ converges to $0$. Indeed, otherwise we would have a sequence  $t_k\rightarrow \infty$ such that $I_{Stein}(\mu_{t_k}|\pi)> \varepsilon>0$. Moreover, since  $I_{Stein}(\mu_{t}|\pi)$ has bounded time derivative, it is uniformly $L$-Lipschitz. There exists a sequence of intervals $I_k$ of length $\frac{ \varepsilon}{L}$ centered at $t_k$ (that we can assume disjoints without loss of generality since $t_k\to \infty$), such that $I_{Stein}(\mu_{t}|\pi)\geq \frac{ \varepsilon}{2}$ for every $t \in I_k$. Now, integrating the dissipation~(see Proposition~\ref{prop:dissipation}) over $\R^{+}$ we get:
\begin{align}
\KL(\mu_0|\pi) - \KL(\mu_t|\pi) = \int_0^{t} I_{Stein}(\mu_{s}|\pi)ds\geq \sum_{k, t_k\leq t} \frac{ \varepsilon^2}{2L}.
\end{align}
The above sum diverges as $t$ goes to infinity since $t_k\rightarrow +\infty $. This is in contradiction with $\KL(\mu_0|\pi)<\infty$. Hence, $I_{Stein}(\mu_t|\pi) \rightarrow 0$.

% The rest of the proof follows \cite[Theorem 2.8]{lu2019scaling}. Indeed, to show that $\mu_t$ converges weakly towards $\pi$, we only need to show that every sequence $\mu_{t_k}$ has a further subsequence that converges weakly towards $\pi$. By compactness of the sublevel set $C$, we have that $\mu_{t_k}$ is tight. 
% Hence, there exists a subsequence $\mu_{u(t_k)}$  that converges  weakly to some $\mu*$. 
% Moreover, by continuity of $\rho\mapsto I_{Stein}(\rho|\pi)$ under the weak topology, 
% we directly have that $I_{Stein}(\mu*|\pi) = \lim_k I_{Stein}(\mu_{t_k}|\pi) = 0$. Since $I_{Stein}$ is separating this implies that $\mu*= \pi$. 
% Hence, we can conclude using  \cite[Theorem 2.6]{Billingsley:1999}.

\subsection{Proof of \Cref{prop:descent}}\label{sec:proof_descent}

We justify each step of the sketch of the proof of \Cref{sec:descent_lemma_Rd}. %For notation purposes we will omit the injection $\iota$ when taking the $L^2$ scalar product with an element of $\cH$.

Consider $n \geq 0$ fixed and $\gamma \leq \frac{\alpha-1}{\alpha BC^{\frac{1}{2}} }$. Denote $g = P_{\mu_n}\nabla \log\left(\frac{\mu_n}{\pi}\right)$ and  for every $t \in [0,\gamma]$,  $\phi_t = (I - t  g)$. Denote $\rho_t = \phi_{t\#} \mu_n$. Then, $\rho_0 = \mu_n$ and $\rho_{\gamma} = \mu_{n+1}$.

\begin{lemma}\label{lem:inequalities}
	Suppose \Cref{ass:k_bounded} holds, i.e. the kernel and its gradient are bounded by some positive constant $B$. Then for any $x\in \X$:
	\begin{align}
&	\|g(x)\| \leq B I_{Stein}(\mu_n|\pi)^{\frac{1}{2}}\\
&	\Vert J g(x)\Vert_{HS}\leq B I_{Stein}(\mu_n|\pi)^{\frac{1}{2}} 
	\end{align}
	
\end{lemma}
\begin{proof}
	This is a consequence of the reproducing property and Cauchy-Schwarz inequality in the RKHS space. Let $g'=S_{\mu_n}\nabla \log(\frac{\mu_n}{\pi})$, hence for any $x\in \X$, $g(x)=g'(x)$ and:
	\begin{equation*}
	\|g(x)\|^2=\sum_{i=1}^d \ps{k(x,.),g'_i}^2_{\cH_0}\le \|k(x,.)\|_{\cH_0}^2 \|g'\|_{\cH}^2\le B^2 I_{Stein}(\mu_n|\pi).
	\end{equation*}
	Similarly:
	\begin{align*}
	\|Jg(x)\|_{HS}^2&=\sum_{i,j=1}^d \left|\frac{\partial g_i(x)}{\partial x_j} \right|^2=\sum_{i,j=1}^d \ps{\partial_{x_j}k(x,.), g'_i}_{\cH_0}\le \sum_{i,j=1}^d \| \partial_{x_j}k(x,.)\|^2_{\cH_0} \|g'_i\|_{\cH_0}^2\\
	&=\| \nabla k(x,.)\|^2_{\cH}\|g'\|^2_{\cH}\le B^2 I_{Stein}(\mu_n|\pi).
	\end{align*}
\end{proof}

\begin{lemma}\label{lem:diffeo}
	Suppose that~\Cref{ass:k_bounded} and~\Cref{ass:bounded_I_Stein} hold. Then, for any $x\in \X$, $\|t Jg(x)\|_{op} \leq t B\sqrt{C}$ and for every $t < \frac{1}{B\sqrt{C}}$, $\phi_t$ is a diffeomorphism. Moreover, $\|(J\phi_t(x))^{-1}\|_{op} \leq \alpha$. 
	%$\phi_t = I-t\gamma P_{\mu_n}\nabla \log(\frac{\mu}{\pi})(x)$ for $t\in [0,1]$. Fix $\alpha >1$,  Then for $\gamma< \frac{\alpha-1}{\alpha BC^{\frac{1}{2}}} $, $\phi_t$ is a diffeomorphism.
\end{lemma}
\begin{proof}
	First, by \Cref{lem:inequalities} and \Cref{ass:bounded_I_Stein} we have $\|J g(x)\|_{op} \leq \|J g(x)\|_{HS} \leq B\sqrt{C}$.
	If $t < \frac{1}{B\sqrt{C}}$, then $\|t Jg(x)\|_{op} < 1$. Therefore, $J(\phi_t)(x) = I - t Jg(x)$ is regular for every $x$ and $\phi_t$ is a diffeomorphism. Moreover, 
	\begin{equation}
		\|(J\phi_t(x))^{-1}\|_{op} \leq \sum_{k=0}^\infty \|t Jg(x)\|_{op}^k \leq \sum_{k=0}^\infty \|t Jg(x)\|_{HS}^k \leq \sum_{k=0}^\infty (t B \sqrt{C})^k \leq \alpha,
	\end{equation}
	where we used $\gamma \leq \frac{\alpha-1}{\alpha BC^{\frac{1}{2}} }$.
\end{proof}
    %From \Cref{lem:velocity_field}, 
    % Using~\cite[Theorem 5.34]{villani2003topics}, 
    % the velocity field of $\rho_t$ is $v_t \in L^2(\rho_t)$ defined by $v_t(x) = -g(\phi_t^{-1}(x))$. In particular, $v_0 = -g$.

Denote $\varphi(t) = \KL(\rho_t|\pi)$. Using Taylor expansion,
\begin{equation}
    \label{eq:Taylor-integral-remainder-appendix}
    \varphi(\gamma) = \varphi(0) + \gamma \varphi'(0) + \int_{0}^{\gamma} (\gamma - t)\varphi''(t)dt.
\end{equation}
We now identify each term. First, $\varphi(0) = \KL(\mu_n|\pi)$ and $\varphi(\gamma) = \KL(\mu_{n+1}|\pi)$.

To compute $\varphi'(t)$ and $\varphi''(t)$ we have two options. Either we check the assumptions of the optimal transport theorems allowing to apply the chain rule~\cite{villani2003topics,ambrosio2008gradient}, or we do a direct computation. The latter is preferred, although differential calculus over the Wasserstein space is a powerful way to guess the formulas.

\begin{lemma} Denote $w_t(x) = -g(\phi_t^{-1}(x))$. Then,
	\begin{equation*}
		\varphi'(0)=\ps{\nabla_{W_2} \KL(\rho_0|\pi), w_0}_{L^2(\mu_n)} = -I_{Stein}(\mu_n|\pi),
	\end{equation*}
	and,
	\begin{equation*}
		\varphi''(t) = \ps{ w_t,Hess_{\KL(.|\pi)}(\rho_t) w_t}_{L^2(\rho_t)} = \int \left[\|Jg(x)(J\phi_t(x))^{-1}\|_{HS}^2 + \ps{g(x), H_V (\phi_t(x)) g(x)}\right] \mu_n(x)dx.
	\end{equation*}
\end{lemma}
\begin{proof}
We know by \Cref{lem:diffeo} that $\phi_t$ is a diffeomorphism, therefore, $\rho_t$ admits a density given by the change of variables formula:
	\begin{align}
	\rho_t(x) = \vert J \phi_t(\phi_t^{-1}(x))  \vert^{-1} \mu_n(\phi_{t}^{-1}(x)). 
	\end{align}
	Using the transfer lemma with $\rho_t=\phi_{t\#}\mu_n$, $\varphi(t)$ is given by:
	\begin{align*}
	\varphi(t) &= \int \log \left( \frac{\rho_t(y)}{\pi(y)}\right) \rho_t(y)dy\\
	&=  \int \log\left( \frac{\mu_n(x) \vert  J\phi_t(x) \vert^{-1}}{\pi(\phi_t(x))}\right) \mu_n(x)dx.
	\end{align*}
	We can now take the time derivative of $\varphi(t)$ which gives:
	\begin{equation*}
	\varphi'(t) = - \int tr\left(J\phi_t(x)^{-1} \frac{dJ\phi_t(x)}{dt}\right)\mu_n(x)dx - \int \ps{\nabla \log \pi(\phi_t(x)),\frac{d\phi_t(x)}{dt}} \mu_n(x)dx.  
	\end{equation*}
	Hence, we can use the explicit expression of $\phi_t$ to write:
	\begin{equation*}
		\varphi'(t) = \int tr(J\phi_t(x)^{-1} Jg(x))\mu_n(x)dx + \int \ps{\nabla \log \pi(\phi_t(x)), g(x)  }\mu_n(x)dx.  
	\end{equation*}
	The Jacobian at time $t=0$ is simply equal to the identity since $\phi_0=I$. It follows that $tr(J\phi_0(x)^{-1} Jg(x)) =tr(Jg(x))= div(g)(x)$ by definition of the divergence operator. Using an integration by parts:
	\begin{align*}
		\varphi'(0) &= -\int \left[-div( g )(x)- \ps{\nabla \log \pi(x),g(x)}\right]\mu_n(x)dx\\
	&= -\int\ps{ \nabla \log\left(\frac{\mu_n}{\pi}\right)(x), g(x)}\mu_n(x)dx
	= -I_{Stein}(\mu_n|\pi).
	\end{align*}
	Now, we prove the second statement. First, 
	\begin{align*}
	\varphi''(t)= \int \left[tr((Jg(x)(J\phi_t(x))^{-1})^2) + \ps{g(x), H_V (\phi_t(x)) g(x)}\right] \mu_n(x)dx. 
	\end{align*}
	Since $Jg(x)$ and $J\phi_t(x)$ commutes, $tr((Jg(x)(J\phi_t(x))^{-1})^2) = \|Jg(x)(J\phi_t(x))^{-1}\|_{HS}^2$. Moreover, using the chain rule, 
	\begin{equation}
		-J w_t(x) = J (g \circ \phi_t^{-1})(x) = Jg(\phi_t^{-1}(x)) J(\phi_t^{-1})(x) = Jg(\phi_t^{-1}(x)) (J\phi_t)^{-1}(\phi_t^{-1}(x)).
	\end{equation}
		Therefore, $\|Jg(x)(J\phi_t(x))^{-1}\|_{HS}^2 = \|J w_t(\phi_t(x))\|_{HS}^2$, which proves the second part of the second statement. Using the transfer lemma,
		\begin{align*}
			\varphi''(t)&= \int \left[\|J w_t(y)\|_{HS}^2 +\ps{ w_t(y),H_V (y) w_t(y)}\right] \rho_t(y)dy\\
			&=\ps{ w_t,Hess_{\KL(.|\pi)}(\rho_t) w_t}_{L^2(\rho_t)},
			\end{align*}
			which concludes the proof.
\end{proof}
Denote
\begin{equation*}
        \psi_1(t) = \int \left[\|Jg(x)(J\phi_t(x))^{-1}\|_{HS}^2\right] \mu_n(x)dx \; \text{ and } \; \psi_2(t) = \int \ps{g(x), H_V (\phi_t(x)) g(x)}\mu_n(x)dx.
		\end{equation*}
	Then, $\varphi''(t) = \psi_1(t) + \psi_2(t)$.
% Second, using the chain rule in the Wasserstein space, 
% \begin{multline*}
% \varphi'(0) = \ps{\nabla_{W_2} \KL(\rho_0|\pi),v_0}_{L^2(\rho_0)} = \ps{P_{\rho_0}\nabla_{W_2} \KL(\rho_0|\pi),v_0}_{\cH}
%  = -\|g\|_{\cH}^2=-I_{Stein}(\mu_n|\pi)
% \end{multline*}
%  using $v_0 \in \cH$. 
% Finally, recall that the Hessian of $\KL(.|\pi)$ at $\mu$ is the operator $H_{\KL(.|\pi)}(\mu) : L^2(\mu) \to L^2(\mu)$ defined by 
% \begin{equation}
%     \ps{v, H_{\KL(.|\pi)}(\mu) v}_{\mu} = E_{X \sim \mu} \left[\|J v(X)\|_{HS}^2 + \ps{v(X), H_V(X) v(X)}\right],
% \end{equation}
% where $H_V$ is the Hessian matrix of $V = -\log(\pi)$ and $Jv$ the Jacobian matrix of $v \in L^2(\mu)$. Hence,
% \begin{equation}
%     \varphi''(t) %= \ps{v(t),H_{\cF}(\mu(t)) v(t)}_{\mu(t)} 
%     = E_{x_t \sim \rho_t} \left[ \|J v_t(x_t)\|_{HS}^2 + \ps{v_t(x_t), H_V(x_t) v_t(x_t)}\right].
% \end{equation}
We bound $\psi_1$ and $\psi_2$ separately. First, since the potential $V$ is $M$-smooth, 
\begin{align*}
	\psi_2(t) \leq M\int \|g(x)\|^2 \mu_n(x)dx \leq M B^2 I_{Stein}(\mu_n|\pi),
\end{align*}
by using \Cref{lem:inequalities}.
%$\ps{v,H_V v} \leq M \|v\|^2$. Therefore, 
% \begin{align*}
%     E_{x_t\sim \rho_t} \left[ \ps{ v_t(x_t), H_V(x_t) v_t(x_t) } \right] &= E_{x_0\sim \rho_0} \left[ \ps{v_t(\phi_t(x_0)), H_V(\phi_t(x_0)) v_t(\phi_t(x_0))}\right] \\
%     &= E_{x_0\sim \rho_0} \left[\ps{v_0, H_V(\phi_t(x_0)) v_0}\right]\\
%     &\leq M \|g\|_{L^2(\mu_n)}^2 \le M B^2 I_{Stein}(\mu_n|\pi).
% \end{align*}
%by domination of the RKHS norm over the $L^2$ norm.\aknote{add this in background}
% For the second term, note that by the chain rule :
% \begin{equation}
% -J v_t(x) = J (g \circ \phi_t^{-1})(x) = Jg(\phi_t^{-1}(x)) J(\phi_t^{-1})(x) = Jg(\phi_t^{-1}(x)) (J\phi_t)^{-1}(\phi_t^{-1}(x)),
% \end{equation}
% therefore $E_{x_t \sim \mu_t} \left(\|J v_t(x_t)\|_{HS}^2\right) = E_{x_0 \sim \rho_0} \left(\|Jg(x_0) J(\phi_t)^{-1}(x_0)\|_{HS}^2\right)$. Sufficiently, we shall bound the term $\|J(g)(x) J(\phi_t)^{-1}(x)\|_{HS}^2$. First,
Now, we bound $\psi_1(t)$ using \Cref{lem:diffeo} and \ref{lem:inequalities}:
\begin{equation}
    \|Jg(x)(J\phi_t(x))^{-1}\|_{HS}^2 \leq \|Jg(x)\|_{HS}^2 \| (J\phi_t(x))^{-1}\|_{op}^2 \leq \alpha^2 B^2 I_{Stein}(\mu_n|\pi).
\end{equation}

% Then, $J\phi_t(x) = I - t Jg(x)$ and, since $t < \frac{1}{B\sqrt{C}}$, we have $t\|Jg(x)\| < 1$. Therefore, 
% \begin{equation}
% \|(I - t Jg(x))^{-1}\|_{op} \leq \sum_{k=0}^{+\infty} \|t Jg(x)\|_{op}^k = \frac{1}{1-t\|Jg(x)\|_{op}} \leq \frac{1}{1-t\|Jg(x)\|_{HS}} \leq \frac{1}{1-\gamma B\sqrt{C}}.
% \end{equation}
% Therefore, 
% \begin{equation}
%     \|Jg(x) (J\phi_t)^{-1}(x)\|_{HS}^2 \leq B^2 I_{Stein}(\mu_n|\pi) \frac{1}{(1-\gamma B\sqrt{C})^2}.
% \end{equation}
% Hence, if $\alpha > 1$ and $\gamma \leq \frac{\alpha - 1}{\alpha B \sqrt{C}}$, then $\frac{1}{1-\gamma B\sqrt{C}} \leq \alpha$ and 
% \begin{equation}
%     \|Jg(x) J(\phi_t)^{-1}(x)\|_{HS}^2 \leq \alpha^2 B^2 I_{Stein}(\mu_n|\pi).
% \end{equation}
Finally, $\varphi''(t) \leq (\alpha^2 + M)B^2 I_{Stein}(\mu_n|\pi)$. Plugging into~\eqref{eq:Taylor-integral-remainder-appendix} gives the result.

\subsection{About combining the Stein log Sobolev assumption and a descent lemma}\label{sec:discussion_KL_rates}

%\begin{theorem}\label{th:cv_with_descent} 
%	Let $\alpha>1$ and $\gamma \leq \min \left(\frac{\alpha-1}{\alpha BC^{\frac{1}{2}} }, \frac{2}{(\alpha^2+M)B^2}\right) $. 
%	Under the assumptions of \Cref{prop:descent}, if $\pi$ satisfies the Stein log Sobolev inequality \eqref{eq:stein_log_sobolev} with $\lambda>0$ at all iterations $n\ge 0$, then with $c_{\gamma}=\gamma \left(1- \gamma \frac{(\alpha^2 + M)B^2}{2}\right) $,
%	\begin{equation}
%		\KL(\mu_n|\pi)\le (1-2c_{\gamma}\lambda)^{n}\KL(\mu_0|\pi).
%	\end{equation}
%	Hence, if $2c_{\gamma}\lambda<1$, we obtain exponential convergence.
%\end{theorem}
An insight deriving from the optimization perspective is that linear rates could be obtained by combining a descent result such as in \cref{prop:descent} and a Polyak-Lojasiewicz condition on the objective function \cite{karimi2016linear}.  In our case, the latter condition corresponds to the Stein log Sobolev inequality from~\cite{duncan2019geometry}.
	Using the descent \Cref{prop:descent} and the Stein log Sobolev inequality \eqref{eq:stein_log_sobolev} we would have that:
	\begin{equation*}
		\KL(\mu_{n+1}|\pi)-\KL(\mu_n|\pi)\le -c_{\gamma}  I_{Stein}(\mu_n|\pi)\le -2c_{\gamma}\lambda \KL(\mu_n|\pi),
	\end{equation*}
 hence $\KL(\mu_{n+1}|\pi)\le (1-2c_{\gamma}\lambda)\KL(\mu_n|\pi)$ which would result by iteration in a linear rate for the KL objective. \textit{However}, it seems impossible to combine the assumptions needed for our descent lemma, in particular about the kernel and its derivative being bounded, while being able to asssume that the Stein log Sobolev inequality holds. It seems that no such $\pi$ and $k$ exist  (at least for $\mathcal{X}=\mathbb{R}^d$). Given that both the kernel and its derivative are bounded, equation 
 \[\int \sum_{i=1}^d [
 (\partial_iV (x))^2
 k(x, x) -\partial_i V(x)(
 \partial_
 i^1 k(x, x) + \partial^
 2_i k(x, x))
 +  \partial_i^{1}
 \partial_i^{2} k(x, x)]d\pi(x) < \infty\]
 reduces to a property on $V$ which, as far as we can tell, always holds; and this implies that Stein LSI does not hold (see \cite[Lemma 36]{duncan2019geometry}). For instance, even when $V=-\log (\mathrm{cauchy})$ or $V=-\log (\mathrm{student})$ the negative log densities of a Cauchy or Student distribution,  we quickly find that the resulting expectations are bounded hence Stein LSI does not hold.

\subsection{Proof of \Cref{prop:finite_particle}}\label{sec:proof_particles}

%We state here additional assumptions on the kernel $k$ and target distribution $\pi$ needed for our results.

Introduce the system of $N$ \textit{ independent} particles:
\begin{equation}
\bar{X}_{n+1}^{i}=\bar{X}_{n}^{i}-\gamma P_{\mu_n}\nabla \log\left(\frac{\mu_n}{\pi}\right)(\bar{X}_n^{i}), \quad \bar{X}_{0}^{i}\sim \mu_0.
\end{equation}
By definition, $(\bar{X}_{n}^{i})_{i=1}^N$ are i.i.d. samples from $\mu_n$.
Let 
%\begin{equation*}
$c_n=\left(\frac{1}{N}\sum_{i=1}^{N}\E[\| \bar{X}_{n}^{i} - X_n^{i} \|^2]\right)^{\frac{1}{2}}$. 
%\end{equation*}
Notice that $c_n\ge W_2(\mu_n,\hat{\mu}_n)$ since the 2-Wasserstein is the infimum over the couplings between $\mu_n$ and $\bar{\mu}_n$. At time $n+1$, we have: 
\begin{align*}
c_{n+1}&= \frac{1}{\sqrt{N}} \left(\sum_{i=1}^{N} \E[\| X_{n+1}^{i}- \bar{X}_{n+1}^{i}\|^2]\right)^{\frac{1}{2}}\\
&= \frac{1}{\sqrt{N}}\left(\sum_{i=1}^{N}\E[\|  X_{n}^{i}- \bar{X}_n^{i} -\gamma(P_{\hat{\mu}_n} \nabla \log (\frac{\hat{\mu}_n}{\pi})(X_n^{i}) - P_{\mu_n} \nabla \log (\frac{\mu_n}{\pi})(\bar{X}_n^{i})  )  \|^2 ]\right)^{\frac{1}{2}}\\
& \le c_n + \frac{\gamma}{\sqrt{N}}\left( \sum_{i=1}^{N} \E[\| P_{\hat{\mu}_n} \nabla \log (\frac{\hat{\mu}_n}{\pi})(X_n^{i}) - P_{\mu_n} \nabla \log (\frac{\mu_n}{\pi})(\bar{X}_n^{i}) \|^2] \right)^{\frac{1}{2}}
\end{align*}
By introducing $\bar{\mu}_n$ the empirical distribution of the particles $(\bar{X}_n^{i})_{i=1}^{N}$, the second term on the right hand side can be decomposed as the square root of the sum of two terms $A$ and $B$ defined as:
\begin{align*}
A& =\sum_{i=1}^{N} \E[\| P_{\hat{\mu}_n} \nabla \log (\frac{\hat{\mu}_n}{\pi})(X_n^{i}) - P_{\bar{\mu}_n} \nabla \log (\frac{\bar{\mu}_n}{\pi})(\bar{X}_n^{i}) \|^2]\\
B&=\sum_{i=1}^{N} \E[\| P_{\bar{\mu}_n} \nabla \log (\frac{\bar{\mu}_n}{\pi})(\bar{X}_n^{i}) - P_{\mu_n} \nabla \log (\frac{\mu_n}{\pi})(\bar{X}_n^{i}) \|^2]
\end{align*}
%Under \Cref{ass:V_bounded},\ref{ass:k_bounded}, \ref{ass:V_Lipschitz},\ref{ass:k_lipschitz}, the potential $V$ as well as the kernel $k$ and its gradient $\nabla_x k(x,.)$ are bounded and Lipschitz; hence by \Cref{lem:lipschitz_product} we have $b(x,y)=\nabla \log \pi(x)$
By using \Cref{lem:Pmu_lipschitz}, the map $(z,\mu)\mapsto P_{\mu}\nabla \log(\frac{\mu}{\pi})(z)$ is $L$-Lipschitz and we can bound the first term as follows : 
\begin{align*}
A& \le \sum_{i=1}^{N} \E[\| P_{\hat{\mu}_n} \nabla \log (\frac{\hat{\mu}_n}{\pi})(X_n^{i}) - P_{\hat{\mu}_n} \nabla \log (\frac{\hat{\mu}_n}{\pi})(\bar{X}_n^{i}) \|^2] + \sum_{i=1}^{N} \E[\| P_{\hat{\mu}_n} \nabla \log (\frac{\hat{\mu}_n}{\pi})(\bar{X}_n^{i}) - P_{\bar{\mu}_n} \nabla \log (\frac{\bar{\mu}_n}{\pi})(\bar{X}_n^{i}) \|^2]\\
&\le \sum_{i=1}^N L^2 \E[\|X_{n}^{i}- \bar{X}_n^{i}\|^2] + \sum_{i=1}^{N}L^2 \E[W_2^2(\hat{\mu}_n,\bar{\mu}_n)]\\
&= NL^2c_n^2 + NL^2 \E[W_2^2(\hat{\mu}_n,\bar{\mu}_n)].
\end{align*}
Hence, 
\begin{equation*}
A^{\frac{1}{2}}\le L\sqrt{N}(c_n+\E[W_2^2(\hat{\mu}_n,\bar{\mu}_n)]^{\frac{1}{2}})\le 2L \sqrt{N}c_n.
\end{equation*}
The second term can be bounded as: 
\begin{align*}
B&=\sum_{i=1}^{N} \E[\|\frac{1}{N}\sum_{i=1}^{N}( b(\bar{X}_n^{j},\bar{X}_n^{i}) -\int b(x,\bar{X}_n^{i})d\mu_n(x)) \|^2]\\
&=\sum_{i=1}^{N}\frac{1}{N^2}\sum_{j=1}^{N}\E[\| b(\bar{X}_n^{j},\bar{X}_n^{i} ) -\int b(x,\bar{X}_n^{i})d\mu_n(x)  \|^2 ]\\
&\le \sum_{i=1}^N \frac{1}{N^2}\sum_{j=1}^{N} L^2\E[ \| \bar{X}_n^{j} - \int x d\mu_n(x)\|^2]\\
%\le L^2 \E [W_2^2(\bar{\mu}_n,\mu_n)]
&\le L^2 var(\mu_n)
\end{align*}
%where we have used successively that the $(\bar{X}_n^{j})_{j=1}^{N}$ are i.i.d. and \Cref{lem:Pmu_lipschitz}. 
%by using \Cref{lem:Pmu_lipschitz}. 
by using \Cref{cor:blipschitz}. 
Hence,
\begin{equation*}
B^{\frac{1}{2}}\le L \sqrt{var(\mu_n)},
\end{equation*}
and we get the recurrence relation for $c_n$:
\begin{align*}
c_{n+1}&\le c_n + \frac{\gamma}{\sqrt{N}}(A+B)^{\frac{1}{2}}\\
&\le c_n + \frac{\gamma}{\sqrt{N}}(2L\sqrt{N}c_n + L\sqrt{var(\mu_n)})\\
&\le c_n(1+ 2\gamma L) + \frac{\gamma L}{\sqrt{N}}\sqrt{var(\mu_n)}\\
&\le \frac{1}{2}\left(\frac{1}{\sqrt{N}} \sqrt{var(\mu_0)}e^{LT} \right)(e^{2LT}-1)
\end{align*}
where the last line uses \Cref{lem:Control_variance}.

\begin{lemma}\label{lem:Control_variance}
	Consider an initial distribution
	$\mu_{0}$ with finite variance. Define the sequence of probability distributions $\mu_{n+1}=(I-\gamma P_{\mu_n}\nabla \log (\frac{\mu_n}{\pi}))_{\#}\mu_n$.	Under \Cref{ass:k_bounded},\ref{ass:V_Lipschitz},\ref{ass:V_bounded}, \ref{ass:k_lipschitz},  the variance of
	$\mu_{n}$ satisfies for all $T>0$ and $n\leq\frac{T}{\gamma}$ the following inequality:
	\[
	var(\mu_{n})^{\frac{1}{2}}\le var(\mu_{0})^{\frac{1}{2}}e^{TL}
	\]
	for $L$ a constant depending on $k$ and $\pi$.
\end{lemma}
\begin{proof} Denote by $x$ and $x'$ two independent samples from $\mu_n$. We have : 
	\begin{align*}
	var(\mu_{n+1})^{\frac{1}{2}} 
	& =\left(\mathbb{E}\left[\left\Vert x -\mathbb{E}\left[x'\right] -\gamma P_{\mu_n}\nabla \log (\frac{\mu_n}{\pi})(x)+\gamma\mathbb{E}\left[P_{\mu_n}\nabla \log (\frac{\mu_n}{\pi})(x')\right]\right\Vert^2\right]\right)^{\frac{1}{2}}\\
	& \leq var(\mu_{n})^{\frac{1}{2}}+\gamma\left(\mathbb{E}\left[\left\Vert P_{\mu_n}\nabla \log (\frac{\mu_n}{\pi})(x)-\mathbb{E}\left[P_{\mu_n}\nabla \log (\frac{\mu_n}{\pi})(x')\right]\right\Vert^{2}\right]\right)^{\frac{1}{2}}\\
	& \leq var(\mu_{n})^{\frac{1}{2}}+\gamma L\mathbb{E}_{x,x'\sim\mu_{n}}\left[\left\Vert x-x'\right\Vert^{2}\right]^{\frac{1}{2}}\\
	& \leq var(\mu_{n})^{\frac{1}{2}}+\gamma L var(\mu_{n})^{\frac{1}{2}}
	\end{align*}
	The second and last lines are obtained using a triangular inequality while the third line uses that $x\mapsto P_{\mu_n}\nabla \log (\frac{\mu_n}{\pi})(x)$ is $L$-Lipschitz by  \Cref{lem:Pmu_lipschitz}. We then conclude using \Cref{lem:Discrete-Gronwall-lemma}.
\end{proof}

\begin{lemma}
	\label{lem:Discrete-Gronwall-lemma}[Discrete Gronwall lemma]
	Let $a_{n+1}\leq(1+\gamma A)a_{n}+b$ with $\gamma>0$, $A>0$,
	$b>0$ and $a_0=0$, then: 
	\[
	a_{n}\leq\frac{b}{\gamma A}(e^{n\gamma A}-1).
	\]
\end{lemma}
\begin{proof}
	Using the recursion, it is easy to see that for any $n>0$:
	\[
	a_n \leq (1+\gamma A)^n a_0 + b\left(\sum_{i=0}^{n-1}(1+\gamma A )^{k}\right) 
	\]
	One concludes using the identity $\sum_{i=0}^{n-1}(1+\gamma A )^{k} =\frac{1}{\gamma A}((1+\gamma A)^{n} -1)$ and recalling that $(1+\gamma A)^{n} \leq e^{n\gamma A}$.
\end{proof}

\section{Auxiliary results}\label{sec:auxiliary_results}

\begin{comment}
\begin{lemma}\label{lem:lipschitz_product}
	Let $f:\X \to \R$, $g:\X\to \R$ respectively $C_f$ and $C_g$ Lipschitz. Asumme $\exists D>0$, such that for any $x\in \X$, $|f(x)|\le D $ and $|g(x)|\le D$. Then, $f+g$ is Lipschitz with constant $D(C_f+C_g)$.
\end{lemma}
\begin{proof}
	For any $x,x'\in \X^2$,
	\begin{equation*}
	|f(x)g(x)-f(x')g(x')|\le |f(x)g(x)-f(x)g(x')|+|f(x)g(x')-f(x')g(x')|\le D(C_f+C_g)||x-x'\|.
	\end{equation*}
\end{proof}
\end{comment}

%\begin{comment}
\begin{lemma}\label{lem:Pmu_lipschitz}
%When $k$ is in the Stein class of $\pi$\aknote{to complete + write notations for $\R^d$}, we have for any $\mu \in \mathcal{P}_2(\X)$:\begin{equation}\label{eq:svgd_grad_formula}P_{\mu}\nabla \log(\frac{\mu}{\pi})(.)=-\int \nabla \log\pi(x),k(x,.) + \nabla_1 k(x,.) d\mu(x)\end{equation}
Under \Cref{ass:k_bounded},\ref{ass:V_Lipschitz},\ref{ass:V_bounded}, \ref{ass:k_lipschitz}, the map $(z,\mu)\mapsto P_{\mu}\nabla \log(\frac{\mu}{\pi})(z)$ is $L$-Lipschitz with:
\begin{equation}
\Vert P_{\mu}\nabla \log(\frac{\mu}{\pi})(z) - P_{\mu'}\nabla \log(\frac{\mu'}{\pi})(z')\Vert \leq L (\Vert z-z' \Vert + W_2(\mu,\mu'))
\end{equation}
where $L$ depends on $k$ and $\pi$.
\end{lemma}
\begin{proof}
%The expression \ref{eq:svgd_grad_formula} is given by an integration by parts:\begin{align*}P_{\mu}\nabla \log(\frac{\mu}{\pi})(.)&=\int \ps{\nabla \log(\frac{\mu}{\pi})(x),k(x,.)}d\mu(x)\\&=-\int \nabla \log\pi(x)k(x,.)d\mu(x)+\int \ps{\nabla \log\mu(x),k(x,.)}d\mu(x)\\&=-\int \nabla \log\pi(x)k(x,.)d\mu(x)+\int \ps{\nabla \mu(x),k(x,.)}dx\\&=-\int \nabla \log\pi(x)k(x,.)d\mu(x)-\int \nabla_1 k(x,.)d\mu(x)  + [\mu(x)k(x,.) ]_{-\infty}^{+\infty}\\&=-\int \nabla \log\pi(x)k(x,.)+\nabla_1 k(x,.)d\mu(x)\end{align*}since $k$ is in the Stein class of $\pi$.To prove the second statement, 

We will consider an optimal coupling $s$ with marginals $\mu$ and $\mu'$:
\begin{align*}
\Vert P_{\mu}&\nabla \log(\frac{\mu}{\pi})(z) - P_{\mu'}\nabla \log(\frac{\mu'}{\pi})(z')\Vert
= \left\Vert \mathbb{E}_{s}\left[ \nabla \log\pi(x)k(x,z)- \nabla \log\pi(x')k(x',z')) \right] \right.\\
&+\left. \mathbb{E}_{s}\left[\nabla_1 k(x,z)-\nabla_1 k(x',z')\right] \right\Vert\\
& \leq B  \mathbb{E}_{s}\left[ \left\Vert \nabla \log \pi(x)-\nabla \log \pi(x') \right\Vert \right]+
C_V
\mathbb{E}_{s}\left[ \left\Vert k(x,z)-k(x',z') \right\Vert \right] + \mathbb{E}_{s}\left[\left\Vert \nabla_1 k(x,z)-\nabla_1 k(x',z') \right\Vert \right] \\
%& \leq C_V\mathbb{E}_{s}\left[ \left\Vert k(x,z)-k(x',z') \right\Vert \right] + \mathbb{E}_{s}\left[\left\Vert \nabla_1 k(x,z)-\nabla_1 k(x',z') \right\Vert \right] \\
&\leq BM  \mathbb{E}_{s}[\Vert  x-x' \Vert]+
C_V D\left( \Vert  z-z' \Vert + \mathbb{E}_{s}[\Vert  x-x' \Vert]\right) + D\left( \Vert  z-z' \Vert + \mathbb{E}_{s}[\Vert  x-x' \Vert ]\right) \\
&\leq L(\Vert z-z'\Vert + W_2(\mu,\mu'))
\end{align*}
The second line is obtained by convexity while the third one uses \Cref{ass:V_bounded} and  \ref{ass:k_bounded}. The penultimate one uses \Cref{ass:V_Lipschitz} and \ref{ass:k_lipschitz}; finally the last line relies on $s$ being optimal and setting $L=C_V (D+1)+BM$.
\end{proof}
%\end{comment}

\begin{corollary}\label{cor:blipschitz}
	Let $b$ the function defined by $b(x,z)=\nabla \log \pi(x)k(x,z)+ \nabla k(x,z)$. Under the assumptions of \Cref{lem:Pmu_lipschitz}, $b$ is $L$-Lipschitz in its first variable.
\end{corollary}
\begin{proof}
	Notice that $P_{\mu}\nabla \log(\frac{\mu}{\pi})(y)=\E_{x\sim \mu}[b(x,z)]$ for any $\mu\in \cP_2(\X)$ and $z\in \X$. Hence, for any $y,y'\in \X$,
	\begin{equation*}
	|b(y,.)-b(y',.) |\le L W_2(\delta_{y}, \delta_{y'})=L\|y-y'\|. \qedhere
	\end{equation*}
\end{proof}

\begin{lemma}\label{lem:bounded_stein}
	Suppose \Cref{ass:k_bounded} holds, i.e. the kernel and its gradient are bounded by some positive constant $B$. Moreover,  assume that $\nabla \log(\pi)$ is $M$-Lipschitz  and that $\int \Vert x \Vert \mu_n(x)dx $ is uniformly bounded on $n$.	Then $I_{Stein}(\mu_n|\pi)$ remains bounded by some $C>0$, i.e. \Cref{ass:bounded_I_Stein} holds.%\aknote{this also holds under the stronger \Cref{ass:V_bounded}}
\end{lemma}
\begin{proof}
	For any $\mu$, we have :
	\begin{equation*}
	I_{Stein}(\mu|\pi)=\ps{\int \nabla \log \pi(x)k(x,.) + \nabla_1 k(x,.)d\mu(x), \int \nabla \log \pi(y)k(y,.) + \nabla_1 k(y,.)d\mu(y)}_{\cH}
	\end{equation*}
	Using the reproducing property and integration by parts it is possible to write $I_{Stein}(\mu|\pi)$ as:
	\begin{align*}
	I_{Stein}(\mu|\pi) =& \int \nabla_1.\nabla_2 k(x,y)d\mu(x)d\mu(y)\\ &+ \int \ps{\nabla \log \pi(y),\nabla_1 k(x,y)}d\mu(x)d\mu(y)+\int \ps{\nabla \log \pi(x),\nabla_1 k(y,x)}d\mu(x)d\mu(y)\\ 
	&+ \int \ps{\nabla \log \pi(x),\nabla \log \pi(y)} k(x,y)d\mu(x)d\mu(y).
	\end{align*}
	The terms involving the kernel are easily bounded since the kernel is bounded  with bounded derivatives.
		Using that $\nabla \log \pi$ is $M$-Lipschitz, it is easy to see that
		\begin{align}
		\Vert \nabla \log \pi(x) \Vert \leq \Vert \nabla \log \pi(0) \Vert + M \Vert x\Vert.	
		\end{align}
		Using the above inequality, one can directly conclude  that  $\int \Vert x \Vert \mu_n(x)dx $ remains bounded.
\end{proof}

\section{Experiments}\label{sec:experiments}

We downloaded and reused the code (in Python) from \cite{liu2016stein} available at
\url{https://github.com/dilinwang820/Stein-Variational-Gradient-Descent} for our experiments. It implements a toy example with a 1-D Gaussian mixture and a gaussian kernel. In the upper figures, the blue dashed lines are the target density function and the solid green lines are the densities of the (200) particles at different iterations of our algorithm (estimated using kernel density estimator). The lower figures represent the evolution of $I_{Stein}(\hat{\mu}_n|\pi)$ and $\hat{\KL}(\hat{\mu}_n|\pi)$\footnote{where the $\KL(\hat{\mu}_n|\pi)$ is estimated with \textit{scipy.stats.entropy}.} along iterations $n\ge0$. %This code runs in approximately 7 seconds for 500 iterations on a single laptop. 
One can see on the upper figures that the particles recover the target distribution.
On the lower  figure (in log-log scale) , one can see that the average $I_{Stein}$ over $n$ iterations (i.e. $1/n\sum_{k=1}^n I_{Stein}(\hat{\mu}_k|\pi)$) decreases at rate $1/n$ as predicted in \Cref{cor:istein_decreases}. The code to reproduce our results is available : \url{https://github.com/akorba/SVGD_Non_Asymptotic}.
%On the lower right figure (in log scale for the $y$-axis only), one can see that the $\KL$ decreases at linear rate as predicted in \Cref{th:cv_with_descent}. The flattening of the $\KL$ around 200 iterations arises because we use an estimator for this quantity, which becomes visible when it is small.
%The lower bounds on $I_{Stein}$ and $\KL$ is due to numerical approximations (e.g., finite number of particles or using an estimator for the $\KL$).

\begin{figure}[ht!]
	\[
	\begin{array}{cc}
	\includegraphics[width=.5\linewidth]{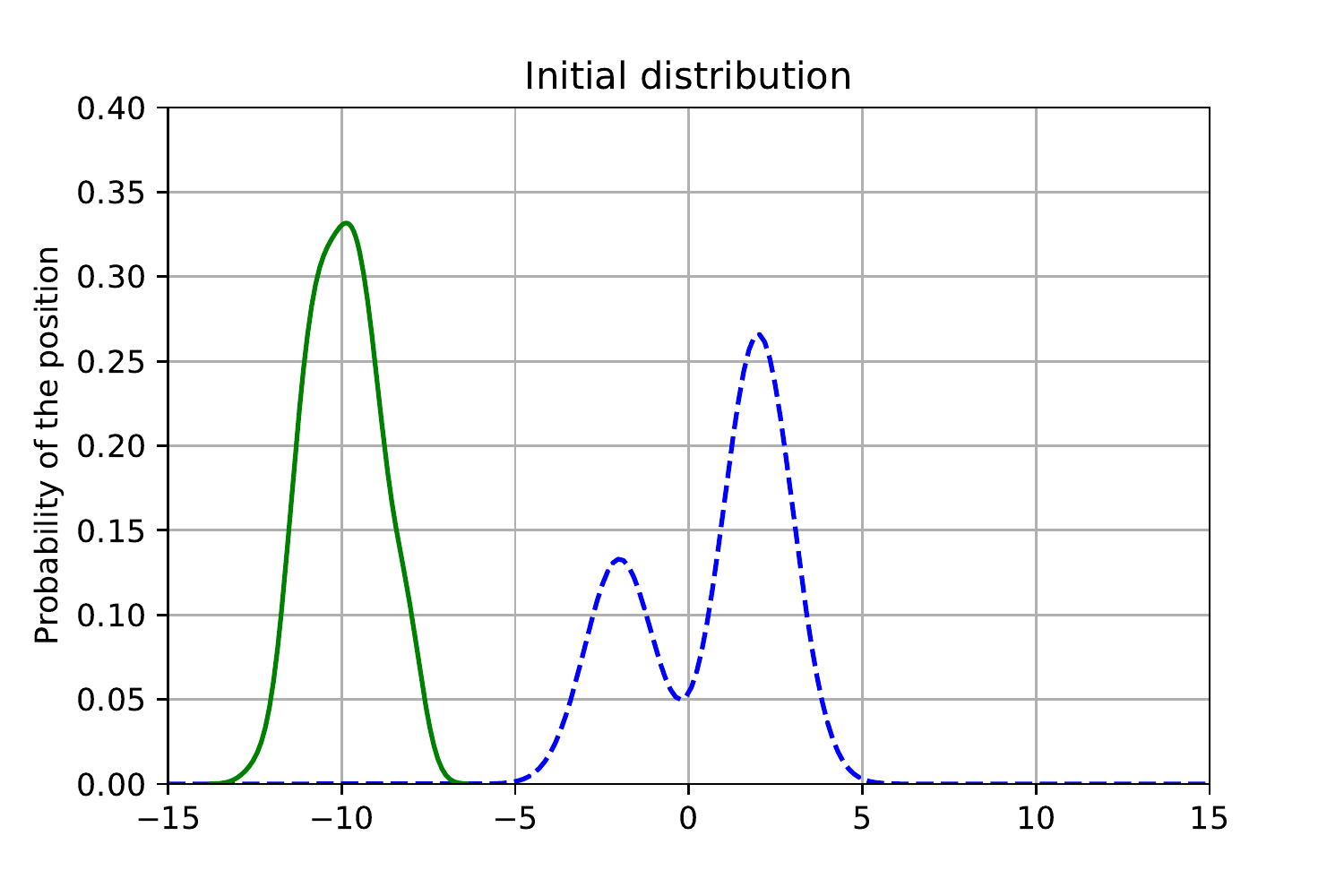} &
	\includegraphics[width=.5\linewidth]{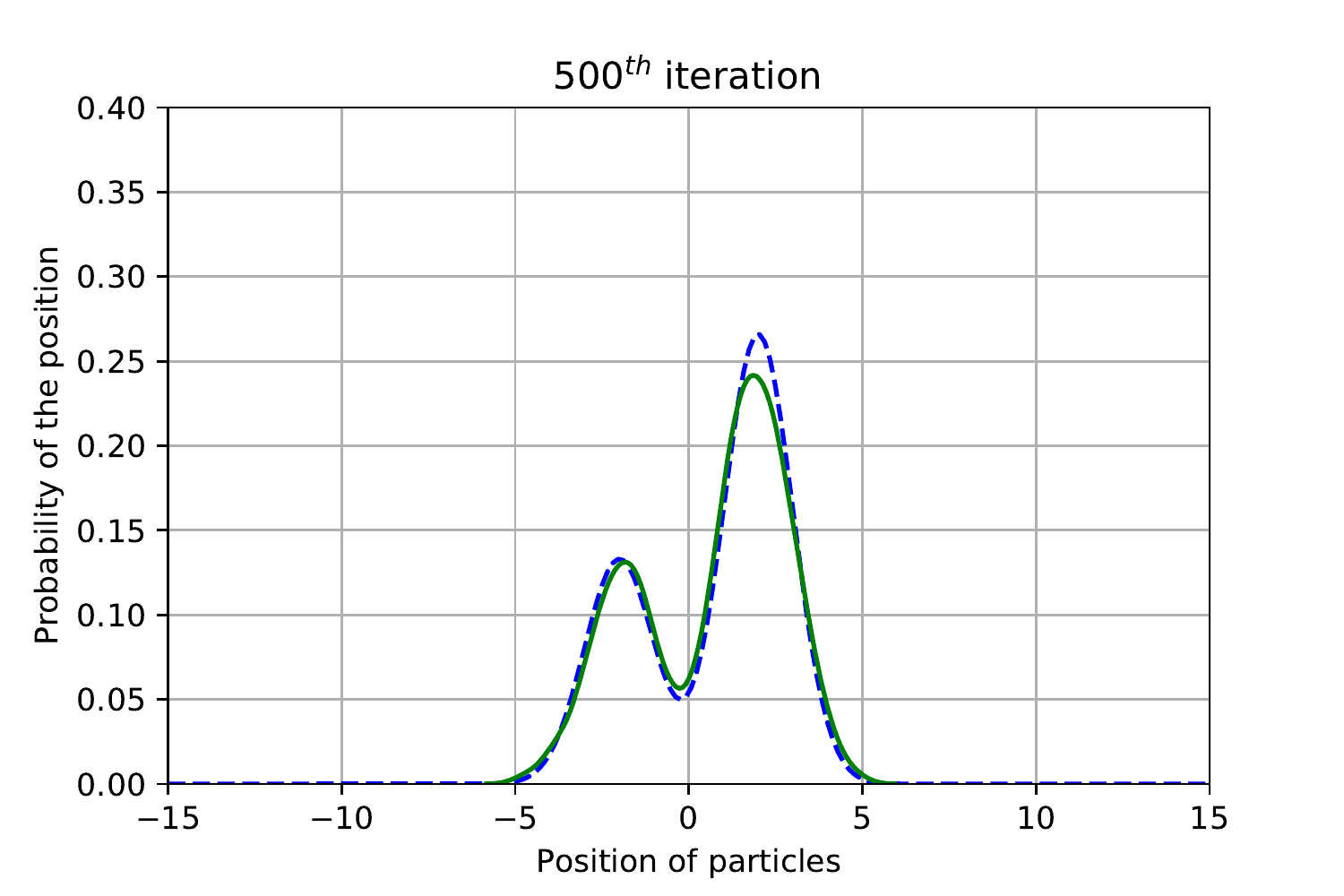}\\
	\end{array}
	\]
	\centering
		\includegraphics[width=.5\linewidth]{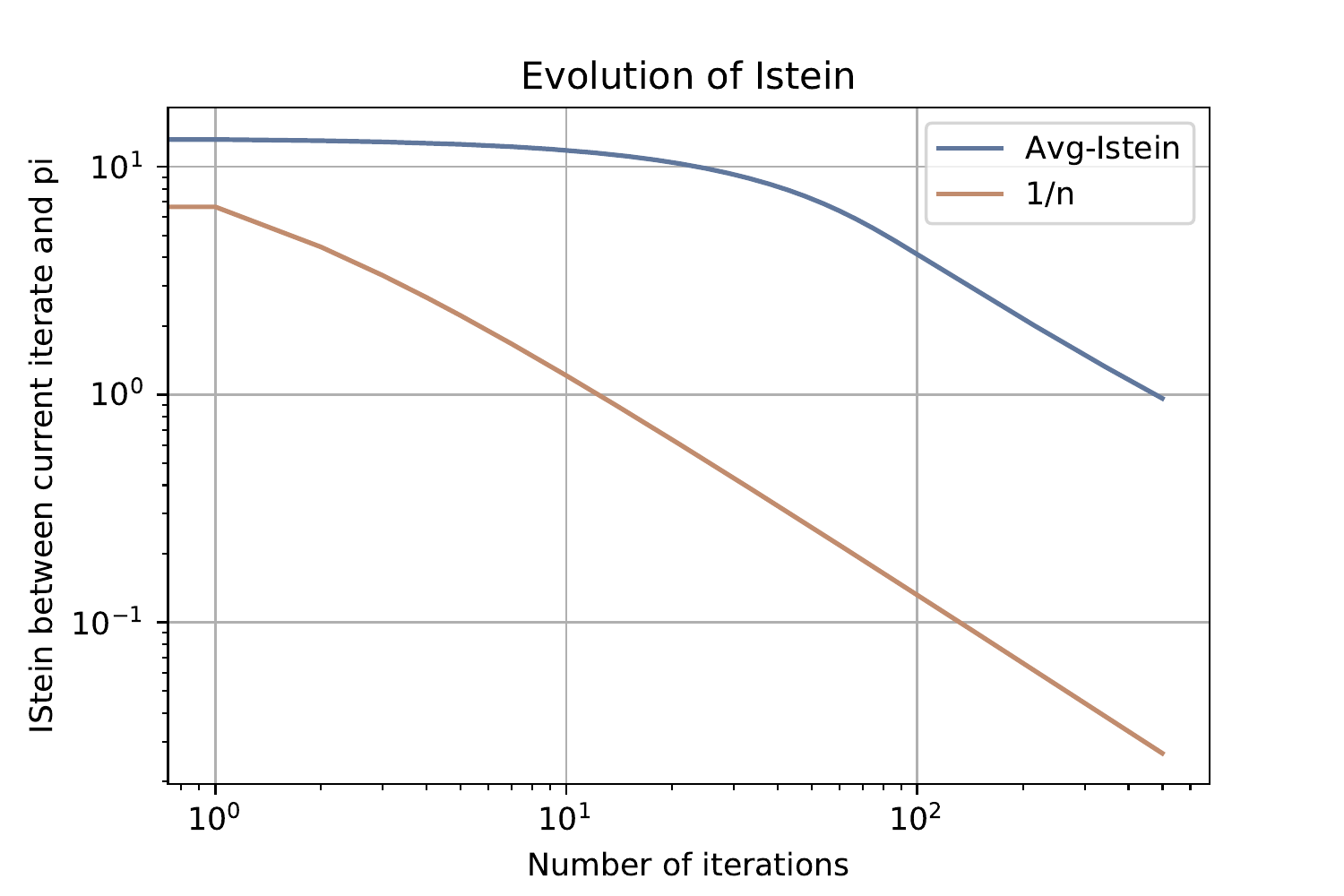}
	\caption{The particle implementation of the SVGD algorithm illustrates the convergence of $I_{Stein}(\mu_n|\pi)$ %and $\KL(\mu_n|\pi)$
		 to $0$.}
	\label{fig:particles}
\end{figure}

\end{document}